\documentclass[twocolumn]{autart}

\usepackage{upgreek}
\usepackage{subfig}
\usepackage{amsmath} 
\usepackage{amssymb}
\usepackage{hyperref}
\usepackage{graphicx,color}
\usepackage{mathrsfs}
\usepackage{cite}
\usepackage{url}
\usepackage{soul}

\usepackage{booktabs}
\usepackage{array}
\usepackage[table]{xcolor}

\usepackage{tikz}
\usetikzlibrary{calc} 
\usetikzlibrary{shapes} 
\usetikzlibrary{chains}
\usetikzlibrary{fit}
\usetikzlibrary{arrows}
\usetikzlibrary{decorations.text} 
\usetikzlibrary{decorations.markings}
\usetikzlibrary{decorations.pathmorphing} 
\usetikzlibrary{shadows}
\usetikzlibrary{patterns}
\usetikzlibrary{matrix}
\usepackage{pgfplots}
\usepackage[europeanresistors]{circuitikz}
\usepackage[outline]{contour} 
\contourlength{1.5pt}

\newtheorem{theorem}{Theorem}
\newtheorem{lemma}{Lemma}

\newtheorem{proof}{Proof}

\newcommand{\R}{\mathbf{R}}
\newcommand{\G}{\mathcal{G}}
\newcommand{\E}{\mathbb{E}}
\newcommand{\V}{\mathcal{V}}
\newcommand{\N}{\mathcal{N}}

\newcommand{\M}{\mathbf{M}}
\newcommand{\A}{\mathbf{A}}
\newcommand{\x}{\mathbf{x}}
\newcommand{\z}{\mathbf{z}}
\newcommand{\g}{\mathbf{g}}
\newcommand{\s}{\mathbf{s}}

\newcommand{\y}{\mathbf{y}}

\newcommand{\Eta}{\boldsymbol{\eta}}

\newcommand{\I}{\mathbf{I}}
\newcommand{\zeros}{\mathbf{0}}

\renewcommand{\H}{\mathbf{H}}

\renewcommand{\P}{\mathbf{P}}
\renewcommand{\S}{\mathbf{S}}

\newcommand{\smin}{\sigma_{\min}}
\graphicspath{{figures/}}

\usepackage{hyperref}

\begin{document}
	
	\begin{frontmatter}
		
		\title{Optimal Spatial-Temporal Triangulation for
			\\Bearing-Only Cooperative Motion Estimation}
		
		\thanks[footnoteinfo]{
        This work was supported by National Natural Science Foundation of China (Grant No. 62473320). The paper was not presented at any conference. Corresponding author: S.~Zhao.
		}
		
		\author[zu,wu]{\quad\quad Canlun Zheng}\ead{zhengcanlun@westlake.edu.cn},
        \author[wu]{Yize Mi}\ead{miyize@westlake.edu.cn},
		\author[wu]{Hanqing Guo}\ead{guohanqing@westlake.edu.cn},
		\author[wu]{Huaben Chen}\ead{chenhuaben@westlake.edu.cn},\newline
		\author[sust,pc]{Zhiyun Lin}\ead{linzy@sustech.edu.cn},
		\author[wu]{Shiyu Zhao}\ead{zhaoshiyu@westlake.edu.cn}
		
		\address[zu]{College of Computer Science and Technology, Zhejiang University, Hangzhou, China}
		\address[wu]{WINDY Lab, Department of Artificial Intelligence, Westlake University, Hangzhou, China}
		\address[sust]{School of System Design and Intelligent Manufacturing, Southern University of Science and Technology, Shenzhen, China}
		\address[pc]{Peng Cheng Laboratory, Shenzhen, China}
		
		
		\begin{abstract}
			Vision-based cooperative motion estimation is an important problem for many multi-robot systems such as cooperative aerial target pursuit. This problem can be formulated as \emph{bearing-only} cooperative motion estimation, where the visual measurement is modeled as a bearing vector pointing from the camera to the target. The conventional approaches for bearing-only cooperative estimation are mainly based on the framework of distributed Kalman filtering (DKF). In this paper, we propose a new optimal bearing-only cooperative estimation algorithm, named spatial-temporal triangulation, based on the method of distributed recursive least squares.
			The design of the algorithm fully incorporates all the available information and the specific triangulation geometric constraint. As a result, the algorithm has superior estimation performance than the state-of-the-art DKF algorithms in terms of both accuracy and convergence speed as verified by numerical simulation.
			We rigorously prove the exponential convergence of the proposed algorithm. Moreover, to verify the effectiveness of the proposed algorithm under practical challenging conditions, we develop a vision-based cooperative aerial target pursuit system, which is the first of such fully autonomous systems up to now to the best of our knowledge.
		\end{abstract}
		
	\end{frontmatter}
	
	\section{Introduction}
	
	This study is motivated by the practical task of vision-based cooperative aerial target pursuit, where multiple micro aerial vehicles (MAVs) use their onboard cameras to detect, locate, and follow an uncooperative target MAV autonomously. This task is inspired by the fascinating bird-catching-bird behaviors in nature \cite{brighton2019hawks}.
	While our previous study has considered the one-to-one case \cite{li2022three,ning2024bearing,ning2024real}, the present works focus on the multiple-to-one case where multiple pursuer MAVs collaborate with each other to pursue one target MAV.
	
	Vision-based cooperative aerial target pursuit is a complex system that involves several interconnected components, such as vision detection, motion estimation, and cooperative control.
	Although there exist studies on cooperative target pursuit \cite{ma2013vision} or pursuit-and-evasion problems \cite{makkapati2019optimal,borra2022reinforcement}, these studies mainly consider control strategies by assuming that the target's motion can be obtained in other ways.
	However, estimating the target's motion is a nontrivial problem. It is also critical to accurately estimate the position and velocity of the target to achieve high-performance autonomous target pursuit.
	
	Vision-based cooperative motion estimation can be formulated as the problem of \emph{bearing-only} cooperative motion estimation, where the visual measurement of a pursuer is modeled as a \emph{bearing vector} pointing from the pursuer to the target. Specifically, when the target has been detected by certain vision algorithms, the bearing vector can be calculated from the target's pixel coordinate and the intrinsic parameters of the camera \cite{vrba2019onboard,li2022three}. By contrast, the range of the target cannot be directly recovered from a single image since the target is uncooperative, and hence, its prior information is unknown. Therefore, it is necessary to study how to estimate the target's motion based merely on bearing measurements obtained from multiple views/images. 
	
	In the bearing-only cooperative estimation problem, the conventional approaches are mainly based on the framework of distributed Kalman filters (DKFs) \cite{nguyen2018instrumental,jiang2023event}.
	However, the DKF-based approaches have several \emph{limitations}.
	First, DKFs are designed for linear systems, but the bearing-only estimation problem is nonlinear. As a result, the estimators designed based on linearized models become sub-optimal and the convergence of the entire networked system becomes difficult to analyze.
	Second, DKFs are designed for general systems and do not consider the unique features of bearing measurements.
	It is important to fully exploit bearing measurements to improve estimation performance further.
	
	In this paper, we adopt the distributed recursive least squares (DRLS) to design the bearing-only cooperative estimators with the inspiration of fusing triangulation into the filter.
	DRLS has some \emph{advantages} compared to DKFs.
	Firstly, it provides greater programmable by allowing customization of the objective function. As a result, we can fully utilize the available information while balancing various performance metrics effectively.
	A well-designed objective function can lead to improved and more interpretable algorithm performance.
	Secondly, a DRLS estimator requires less information shared among the pursuers thereby reducing the communication burden. This aspect is important for inter-MAV systems operating under bandwidth-constrained wireless communication. Thirdly, the convergence of a DRLS estimator can be guaranteed if the objective function is well-designed with appropriate structures and coefficients.
	Theoretical convergence guarantee is crucial for cooperative bearing-only estimation, given the nonlinear and complex networked dynamic problems.
	
	The novelties of this work are summarized as follows.
	
	First, we design a new objective function and derive the corresponding recursive estimation algorithm called \emph{spatial-temporal triangulation} (STT).
	The design of the objective function fully exploits all the available information and specific geometry constraints of bearing-only cooperative estimation.
	The dynamic model is also incorporated into the objective function to ensure smooth estimation.
	The proposed STT algorithm has superior performance than the existing DKF algorithms in terms of both accuracy and convergence speed as verified by simulation results.
	It also achieves comparable performance as the centralized Kalman filter (CKF), which is usually used as a baseline to evaluate distributed estimators.
	
	Second, we rigorously prove that the STT algorithm converges at an exponential rate. It is important to establish such a theoretical convergence guarantee to obtain reliable practical applications. One key technical challenge for the convergence analysis is that the measurement matrix that corresponds to each bearing measurement is rank deficient. We fully exploit the unique features of an orthogonal projection operator (e.g., Lemma~\ref{lemma:exp_lambda_S} to Lemma~\ref{lemma:S_inv}) that is widely used in bearing-based estimation and control tasks \cite{zhao2015bearing,li2022three}, and successfully prove that the expectation of the estimation error converges exponentially.
	
	Third, we develop an autonomous vision-based cooperative aerial target pursuit system to verify the effectiveness of the STT algorithm under practical challenging conditions.
	This system consists of three quadcopter MAVs as the pursers and one quadcopter MAV as the target. All necessary functions are realized in onboard computers including visual detection, motion estimation, and formation control. The three purser MAVs can autonomously detect the target MAV in the images by a well-trained Yolo-based detector, estimate the target's position and velocity by the proposed STT algorithm, and follow the moving target to maintain a pre-specified geometric formation.
	To the best of our knowledge, this system is also the first of such systems reported in the literature up to now.

	\section{Related Work}\label{sec_related_work}
	
	Our work involves two types of problems: distributed cooperative estimation and vision-based target motion estimation. Among them, vision-based target motion estimation can be further divided into two categories: the first category does not require estimating the observer's own state, while the second category requires estimating the observer's own state.

	\subsection{Distributed cooperative estimation}

	The design of distributed estimation algorithms primarily relies on two frameworks:  distributed Kalman filtering (DKF) and distributed recursive least squares (DRLS).
	
	First, DKF is one of the most widely used frameworks. Some classic algorithms designed under this framework include CMKF \cite{olfati2009kalman}, CIKF \cite{olfati2005consensus}, and HCMCI-KF \cite{battistelli2014consensus}. They enhance estimation accuracy by adopting various communication information and cooperative strategies. In the latest HCMCI-KF algorithm \cite{battistelli2014consensus}, sensing nodes can enhance the overall estimation accuracy by exchanging their measurement information, estimated target's state, and estimated covariance matrix. However, it also causes a larger communication overhead and increased computational complexity.
	
	Second, DRLS is also a common framework for designing distributed estimation algorithms \cite{cattivelli2008diffusion,mateos2012distributed}. Compared to DKF, DRLS can be used to achieve a more customized design of the estimator by specifying the objective function. Additionally, DRLS algorithms offer higher computational efficiency and lower communication overhead \cite{mateos2012distributed,rastegarnia2019reduced}. However, the existing work has primarily focused on estimating invariant constants \cite{mateos2012distributed}, leaving limited research on estimating the state of moving targets.

	\subsection{Vision-based estimation without self-localization}

	A common approach for visual-based target motion estimation is to model the vision measurement as an angle or a bearing vector. This involves two main topics: estimation algorithms and observability.

	First, the non-linearity of bearing measurements often leads to instability because the estimated state and measurements are in different symmetric groups. Equivariant systems theory is a method that can address this kind of problem. The core of equivariant systems is that the system maintains consistency under specific group transformations \cite{ng2020equivariant,fornasier2022equivariant,teunissen2003theory}. 
    The bearing-based estimation algorithms usually fall into two categories. The first category employs the polar coordinate form, initially proposed in \cite{hoelzer1978modified}. Taking the extended Kalman filter (EKF) as an example, it typically operates in Euclidean space, but when the state and measurements involve non-Euclidean geometries like rotation, it can be extended to Lie group Kalman filter (LGKF) to handle geometric constraints in state transitions \cite{bonnabel2007left}. This approach not only addresses complex geometric relationships but also maintains system consistency under different group transformations, thereby enhancing estimation accuracy and adaptability. 
    The second category is the pseudo-linear Kalman filtering, which was first proposed in \cite{lingren1978position}. It solves the instability problem by transforming the nonlinear measurement equation into a pseudo-linear measurement equation. Although this transformation results in non-Gaussian noise, the theoretical analysis demonstrated unbiased velocity estimation, and the bias in position estimation can be eliminated by moving observations \cite{aidala1982biased}.

	Second, observability is a fundamental topic in bearing-based estimation tasks. It was proven that the observer needs to have higher-order motion than the target to ensure observability \cite{nardone1981observability, fogel1988nth, ning2024bearing}. However, achieving such higher-order maneuverability relative to the target is challenging in practice, especially when the target has an unknown maneuver motion.
	Cooperative estimation using multiple observers can significantly enhance observability \cite{wang2012novel, he2019trajectory, lin2002comparison}. Triangulation is a classical cooperative estimation method that utilizes multiple bearing measurements to estimate the target's position \cite{klippenstein2007feature}. Further integration with filtering algorithms enables estimating the target's velocity. However, this method has a drawback. That is, it requires two or more measurements at every time step; otherwise, the position cannot be triangulated. However, measurements from neighbors may not be available sometimes in practice due to, for example, unstable vision detection. 
	
	\subsection{Vision-based estimation with self-localization}
	
	The methods mentioned in the previous subsection assume that the observer can determine its own state through external means such as GPS. However, it may be necessary to rely on the observer's self-localization system to determine its own state. 

	The most common self-localization technology is simultaneous localization and mapping (SLAM), which enables the observer to estimate its position within a mapping environment by simultaneously creating a map of the environment. It is worth noting that SLAM only achieves self-localization, and after self-localization, it is still necessary to combine with other algorithms to achieve motion estimation of a target of interest. This task is typically associated with dynamic SLAM \cite{yang2019cubeslam, qiu2019tracking}, where there are moving objects in the environment. In dynamic SLAM, the algorithms for target motion estimation operate on similar principles as those introduced in the previous subsection.

	The basic principle of simultaneous localization and mapping (SLAM) involves obtaining precise nonlinear pose transformations of the camera between different time instances using environmental feature information obtained from images. Pose graph optimization (PGO) is a commonly used method in SLAM for optimizing the pose estimation of robots or drones. The core idea of PGO is to improve the accuracy of global pose estimation by optimizing a graph structure, which can be further adapted into a \emph{nonlinear least squares optimization problem} \cite{yue2023single, carlone2015initialization,deutsch2016framework,sunderhauf2012switchable}. Georgia tech smoothing and mapping (GTSAM) \cite{dellaert2012factor} is a widely used algorithm for solving PGO and is extended to solve nonlinear problems. It can process various measurements, such as position, distance, bearing, and velocity.
	
	\section{Problem Setup and Preliminaries}\label{sec_problem_setup}
	
	\subsection{Problem statement}\label{subsec_problem}
	
	Consider $n$ observers in 3D space, each capable of obtaining a noisy bearing measurement of a target through visual sensing, and sharing information via a wireless communication network.
	Let the current time step be $k$ ($k=1,2,\dots$). The network at time $k$ is described by a undirected graph $\G_k=\{\V,\mathcal{E}_k\}$, where $\V=\{1,\dots,n\}$ is the vertex set and $\mathcal{E}_k\subseteq \mathcal{V} \times \mathcal{V}$ is the edge set. An edge $(i,j)\in\mathcal{E}_k$ indicates that observer $i$ can receive information from observer $j$, implying that observer $j$ is a neighbor of observer $i$ at time $k$. The set of neighbors of $i$ is denoted as $\mathcal{N}_{i,k}=\{j \in \mathcal{V}: (i,j)\in \mathcal{E}_k\}$.
	Due to the potential unreliability of wireless communication, the graph $\G_k$ may vary over time and may not always be connected.
	
	Suppose $\mathbf{p}_k,\mathbf{v}_k\in\mathbb{R}^3$ are the true position and velocity of the target. Let $\s_{i,k}\in\mathbb{R}^3$ be the true position of observer $i$ where $i\in\V$. The unit bearing vector pointing from $\s_{i,k}$ to $\mathbf{p}_k$ is
	\begin{align}
		\g_{i,k} = \frac{\mathbf{p}_k - \s_{i,k}}{r_{i,k}}\in\mathbb{R}^3,
		\label{eq_bearingAccurate}
	\end{align}
	where $r_{i,k} = \|\mathbf{p}_k - \s_{i,k}\|_2$.
	The noisy measurement of $\g_{i,k}$ is
	\begin{align}
		\tilde{\g}_{i,k} & = \R_{i,k}\g_{i,k},\label{eq_measuredBearingVector}
	\end{align}
	where $\R_{i,k}$ is a random rotation matrix that perturbs $\g_{i,k}$.
	The noisy measurement of $\s_i$ is
	\begin{align}
		\tilde{\s}_{i,k} & =\s_{i,k} + \epsilon_{\s_{i,k}},
		\label{eq_measuredposition}
	\end{align}
	where $\epsilon_{\s_{i,k}}$ is a zero-mean Gaussian noise.

	The problem to be solved in this paper is stated as follows. Suppose the $i$th observer 1) can obtain a noisy position measurement $\tilde{\s}_{i,k}$ by, for example, GPS, 2) can obtain a noisy bearing measurement $\tilde{\g}_{i,k}$ of the target from a monocular camera, and 3) certain necessary information transmitted from its neighbors. The aim is to estimate the target's position $\mathbf{p}_k$ and velocity $\mathbf{v}_k$ accurately and promptly.
	
	\subsection{Preliminaries: State and measurement equations}
	
	To estimate the target's motion, we first present the state transition and measurement equations.
	
	The state vector of the target is defined as $\x_k=[\mathbf{p}_k^T,\mathbf{v}_k^T]^T\in\mathbb{R}^6$.
	Suppose the state is governed by a noise-driven double-integrator model:
		\begin{align}
			\x_{k+1} = \A\x_{k} +\mathbf{w}_{k},\label{eq_state_transition}
		\end{align}
	where
	\begin{align}
		\A
		&=\begin{bmatrix}
			\I_3 & \Delta t\I_3 \\
			\zeros_3& \I_3
		\end{bmatrix}\in\mathbb{R}^{6\times 6},
		\label{eq_process_matrix_A}
	\end{align}
	where $\Delta t$ is the sampling time, $\I_3\in\mathbb{R}^{3\times 3}$ is the identity matrix, and $\zeros_3\in\mathbb{R}^{3\times 3}$ is a zero matrix.
	Here, $\mathbf{w}_k\in\mathbb{R}^{6}$ is a zero-mean Gaussian noise: $\mathbf{w}_k \sim~\mathbb{N} \left(0,\mathbf{Q}\right)$. It is notable that the discrete-time model in
	\eqref{eq_state_transition} can be obtained from a continuous-time model \cite{li2003survey}. The specific expressions of $\mathbf{w}_k$ and $\mathbf{Q}$ can be found in \cite{yan2018ridi,aaslund2022covariance} and omitted here. 
	
	Second, we establish the measurement equation.
	The noisy bearing vector in \eqref{eq_measuredBearingVector} is a \emph{nonlinear equation} of the target's state. It is necessary to convert it to a \emph{pseudo-linear equation} to achieve better estimation stability \cite{lin2002comparison}.
	
	To that end, we introduce a useful orthogonal projection operator frequently used in this paper.
	For any unit vector $\g\in \mathbb{R}^3$, define
	\begin{align}
		\P_{\g} = \I_3 - \g\g^T\in\mathbb{R}^{3\times3}.
		\label{eq_P_g_equation}
	\end{align}
	The interpretation of $\P_\g$ is that it projects any vector $\z\in \mathbb{R}^3$ orthogonally onto the plane orthogonal of $\g$. It holds that $\P_\g = \P_\g^T$, $\P_\g^2 = \P_\g$, and $\text{Null}\left(\P_\g\right)=\text{span}\left(\g\right)$.
	The projection matrix $\P_\g$ plays an important role in bearing-related problems \cite{zhao2019bearing}.
	
	With this orthogonal projection matrix, we can convert the nonlinear measurement equation \eqref{eq_measuredBearingVector} to a pseudo-linear one. To that end, we have
	\begin{align}
		\tilde{\g}_{i,k}
		& =\R_{i,k}\g_{i,k}
		=\g_{i,k}+\underbrace{(\R_{i,k}\g_{i,k}-\g_{i,k})}_{\boldsymbol{\mu}_{i,k}}.\label{eq_measuredBearingVector2}
	\end{align}
	In fact, the noise $\boldsymbol{\mu}$ is jointly determined by the noisy vision measurements and the noisy attitude measurement of the observer.
	Multiplying $\P_{\tilde{\g}_{i,k}}$ on both sides of \eqref{eq_measuredBearingVector2} gives
	\begin{align*}
		\zeros=\P_{\tilde{\g}_{i,k}}(\g_{i,k}+\boldsymbol{\mu}_{i,k}),
	\end{align*}
	where the left-hand side is zero because $\P_{\tilde{\g}_{i,k}}\tilde{\g}_{i,k}=\zeros$.
	Substituting \eqref{eq_bearingAccurate} and \eqref{eq_measuredposition} into the above equation yields
	\begin{align*}
		\zeros=\P_{\tilde{\g}_{i,k}}\left(\frac{\mathbf{p}_k - (\tilde{\s}_{i,k} - \epsilon_{\s_{i,k}})} {r_{i,k}} + \boldsymbol{\mu}_{i,k} \right),
	\end{align*}
	which can be reorganized to
	\begin{align*} \P_{\tilde{\g}_{i,k}}\tilde{\s}_{i,k}=\P_{\tilde{\g}_{i,k}}\mathbf{p}_k+\P_{\tilde{\g}_{i,k}}(\epsilon_{\s_{i,k}} +r_{i,k}\boldsymbol{\mu}_{i,k}),
	\end{align*}
	which can be further expressed as
	\begin{align}        \underbrace{\P_{\tilde{\g}_{i,k}}\tilde{\s}_{i,k}}_{\z_{i,k}} = \underbrace{\begin{bmatrix}
				\P_{\tilde{\g}_{i,k}} & \zeros_3
		\end{bmatrix}}_{\H_{i,k}}\x_{k} + \underbrace{\P_{\tilde{\g}_{i,k}}(\epsilon_{\s_{i,k}} +r_{i,k}\boldsymbol{\mu}_{i,k})}_{\boldsymbol{\nu}_{i,k}}.
		\label{eq_pseudo_measurement_equation}
	\end{align}
	As can be seen in \eqref{eq_pseudo_measurement_equation}, the noise $\boldsymbol{\nu}_{i,k}$ reflects the noise $\epsilon_{\s_{i},k}$ in the observer's position and the noise $\boldsymbol{\mu}_{i,k}$. Therefore, the proposed algorithm can handle noisy position measurements and noisy attitude measurements to a certain extent.
	
	Equation~\eqref{eq_pseudo_measurement_equation} is the pseudo-linear measurement equation. It is called pseudo-linear because, although the expression is linear in the state $\x_{k}$, both the state matrix $\H_{i,k}$ and noise $\boldsymbol{\nu}_{i,k}$ are functions of the measurements. Although the noise $\boldsymbol{\nu}_{i,k}$ is not Gaussian anymore, the pseudo-linear equation can achieve better stability than the nonlinear extended Kalman filter \cite{lin2002comparison}.
	
	\subsection{Preliminaries: Some useful results}
	
	Next, we introduce some mathematical preliminaries used throughout this paper.
	
	Let $\smin(\cdot)$ and $\sigma_{\max}(\cdot)$ be the smallest and greatest singular values of a matrix, respectively. Denote $\|\cdot\|$ as the spectral norm of a matrix. For any non-singular square matrix $\A$, we have $\|\A\|=\sigma_{\max}(\A)$ and $\|\A^{-1}\|=1/\smin(\A)$.
	If a matrix is symmetric positive semi-definite, the singular values are equal to its eigenvalues.
	
	For any two symmetric positive semi-definite matrices $\A$ and $\mathbf{B}$, it holds that
	$\smin(\A + \mathbf{B}) \geq \smin(\A) + \smin(\mathbf{B})$.
	If $\A - \mathbf{B} \geq 0$, which means $\A-\mathbf{B}$ is positive semi-definite, then $\smin(\A) \geq \smin(\mathbf{B})$.
	Suppose $\A$, $\mathbf{B}$, and $\mathbf{C}$ are matrices with appropriate dimensions. Let $\otimes$ be the Kronecker product. Then,
	$\A\otimes(\mathbf{B} + \mathbf{C}) = \A\otimes \mathbf{B} + \A \otimes \mathbf{C}$.
	Suppose $\sigma_i(\A)$ and $\sigma_j(\mathbf{B})$ are the $i$th and $j$th singular values of $\A$ and $\mathbf{B}$, respectively. Then, the set of singular values of $\A \otimes \mathbf{B}$ are $\{\sigma_i(\A)\sigma_j(\mathbf{B})\}_{i,j}$. As a consequence,
	$\smin(\A \otimes \mathbf{B})= \smin(\A)\smin(\mathbf{B})$.
	Finally, if $\A$, $\mathbf{C}$, and $\A+\mathbf{C}$ are non-singular, then \cite{henderson1981deriving}
	\begin{align}\label{eq_tool_ucv}
		(\A+\mathbf{C})^{-1} = (\I - (\mathbf{C}^{-1}\A+ \I)^{-1})\A^{-1}.
	\end{align}
	
	\section{Spatial-Temporal Triangulation Algorithm}
	
	This section presents and analyzes a new cooperative bearing-only target motion estimator called \emph{spatial-temporal triangulation} (STT) based on the framework of DRLS. Here, ``spatial'' refers to the information obtained from multiple observers located at different spatial positions at the same moment. Here, ``temporal'' refers to the historical information obtained by each observer at different moments.
	
	\subsection{Objective function}
	
	First of all, we define an objective function that can 1) fully utilize all the information available to each observer, 2) balance different performance metrics such as estimation smoothness and convergence speed, and 3) ensure the resulting algorithm is convergent.
	
	First, suppose $\hat{\x}_{i,t}$ is the estimate of $\x$ by observer $i$ at time $t$. Define the \emph{measurements error} for observer $i$ as
	\begin{align*}
		&J_{\rm meas}(\hat{\x}_{i,t})\\
		& =\sum_{j\in (i\cup\N_{i,t})}
		\alpha_{ij,t}\left(\z_{j,t} - \H_{j,t}\hat{\x}_{i,t}\right)^T\R\left(\z_{j,t} - \H_{j,t}\hat{\x}_{i,t}\right)\\
		&\doteq\sum_{j\in (i\cup\N_{i,t})}\alpha_{ij,t}
		\big\|\z_{j,t} - \H_{j,t}\hat{\x}_{i,t}\big\|^2_{\R}.
	\end{align*}
	The measurement error $J_{\rm meas}(\hat{\x}_{i,t})$ represents the \emph{discrepancy between the estimates and the measurements}.
	More specifically, it involves $(\z_{i,t} - \H_{i,t}\hat{\x}_{i,t})$ and $(\z_{j,t} - \H_{j,t}\hat{\x}_{i,t})$ with $j\in\N_{i,t}$.
	Here, $(\z_{i,t} - \H_{i,t}\hat{\x}_{i,t})$ indicates the deficiency between observer $i$'s measurement and estimate.
	$(\z_{j,t} - \H_{j,t}\hat{\x}_{i,t})$ indicates the deficiency between observer $i$'s estimate and observer $j$'s measurement.
	Here, $\alpha_{ij,t}\ge0$ is a weight and satisfies $\sum_{j\in (i\cup\N_{i,t})}\alpha_{ij,t} = 1$. The weight matrix $\R$ is selected as $ \R=\I_3/\sigma_{\boldsymbol{\nu}}^2$.
	This measurement error incorporates the constraint of triangulation geometry.
	In the ideal case, this measurement error should be zero, meaning that the bearing measurements of different observers intersect at the same point in the 3D space and the estimate coincides with this point.
	
	Second, define the \emph{consensus error} as
	\begin{align*}
		J_{\text{cons}}(\hat{\x}_{i,t})
		&=\sum_{j\in (i\cup\N_{i,t})}\beta_{ij,t}\left(\hat{\x}^{-}_{j,t} - \hat{\x}_{i,t}\right)^T \left(\hat{\x}^{-}_{j,t} - \hat{\x}_{i,t}\right)\\
		&\doteq\sum_{j\in(i\cup\N_{i,t})}\beta_{ij,t}\big\|\hat{\x}^{-}_{j,t}-\hat{\x}_{i,t}\big\|^2,
	\end{align*}
	where $\hat{\x}^{-}_{j,t}=\A\hat{\x}_{j,t-1}$ is the predicted estimate of neighbor $j$.
	The consensus error $J_{\text{cons}}$ represents the \emph{deficiency between the estimates of observer $i$ and its neighbors}.
	Using $\hat{\x}^{-}_{j,t}$ instead of $\hat{\x}_{j,t}$ because the later is still unknown at time $t$.
	Here, $\beta_{ij,t}\ge0$ is a weight and satisfies $\sum_{j\in (i\cup\N_{i,t})}\beta_{ij,t} = 1$.
	In the ideal case, this consensus error should be zero, meaning that all the observers have reached a consensus on the target's state.
	
	The measurement and consensus errors are calculated for a single time step $t$. The overall objective function is the sum of the two errors over the time horizon of $t=1,\dots,k$:
	\begin{align}
		J\left(\{\hat{\x}_{i,t}\}_{t=1}^k\right) =
		& c\sum_{t=1}^{k}\lambda_t^{(k)} J_{\rm meas}(\hat{\x}_{i,t})+\sum_{t=1}^{k}\lambda_t^{(k)}J_{\text{cons}}(\hat{\x}_{i,t}),
		\label{eq_J_x_t_k}
	\end{align}
	where $c $ is a positive weight to balance $J_{\rm meas}$ and $J_{\text{cons}}$. Here, $\lambda_t^{(k)}\in(0,1)$ is a \emph{forgetting factor} that depends on both $t$ and $k$.
	The objective function in \eqref{eq_J_x_t_k} incorporates all the available information and the geometric constraint in the problem of bearing-only cooperative motion estimation.
	Ideally, this objective function should be equal to zero. However, in practice, it is usually nonzero and we need to find the optimal estimate to minimize it.
	
	To that end, we further incorporate the dynamic constraints among $\{\hat{\x}_{i,t}\}_{t=1}^k$. In particular, it follows from the state transition equation that $\hat{\x}_{i,t+1} = \A \hat{\x}_{i,t}$.
	Then, we have $\hat{\x}_{i,t}=\A^{t-k}\hat{\x}_{i,k}$, substituting which into \eqref{eq_J_x_t_k} gives
	\begin{align}
		& J\left(\hat{\x}_{i,k}\right) \nonumber\\
		& =c  \sum_{t=1}^{k}\left(\lambda_t^{(k)}\sum_{j\in (i\cup\N_{i,t})}\alpha_{ij,t}
		\big\|\z_{j,t} - \H_{j,t}\A^{t-k}\hat{\x}_{i,k}\big\|^2_{\R}\right)\nonumber\\
		& \quad+  \sum_{t=1}^{k}\left(\lambda_t^{(k)}\sum_{j\in (i\cup\N_{i,t})}\beta_{ij,t}\big\|\hat{\x}_{j,t}^{-}-\A^{t-k+1}\hat{\x}_{i,k}\big\|^2\right).
		\label{eq_J_x_k}
	\end{align}
	
	\subsection{The proposed STT algorithm}
	
	We next present the algorithm that can find the optimal solution of \eqref{eq_J_x_k}.
	First of all, we design the forgetting factor as
	\begin{align}
		\lambda_t^{(k)}  = \frac{\gamma_2^{k-t}}{\left(\|\A\| + \gamma_1\|\A\|\right)^{k-t+1}},
		\label{eq_value_lambda}
	\end{align}
	where $k$ is the current time step and $t = 1,2,\dots,k$ indicates the historical time steps. Moreover, $\gamma_1$ and $\gamma_2$ are positive constants satisfying $\gamma_1 \geq \gamma_2$.
	It can be verified that $\lambda_{t-1}^{(k)} \leq \lambda_{t}^{(k)}$ since $\|\A\|\ge1$ and $\gamma_1 \geq \gamma_2$.
	The design of $\lambda_t^{(k)}$ in \eqref{eq_value_lambda} ensures the convergence of the resulting estimation algorithm as shown in Section~\ref{sec_covergence_analysis}.
	
	With the above preparation, we are ready to give the recursive STT algorithm that can find the optimal solution of \eqref{eq_J_x_k}.
	The algorithm is described as follows, whereas the detailed derivation is postponed to the next subsection.
	In particular, consider an arbitrary initial estimate $\hat{\x}_{i,0}$ and an initial process matrix estimate $\hat{\M}_{i,0}$, which is a symmetric positive definite.
	The STT algorithm consists of three steps: \emph{prediction}, \emph{innovation}, and \emph{correction}.
	\begin{align}	
		\intertext{\textbf{Prediction:}}
		\hat{\x}^{-}_{i,k} & = \A\hat{\x}_{i,k-1},\label{eq_x_pred}\\
		\hat{\M}^{-}_{i,k} & =  \frac{1}{(1+\gamma_1)\|\A\|}\left(\A\hat{\M}_{i,k-1}\A^T\right)^{-1},\label{eq_M_pred}\\
		\intertext{\textbf{Innovation:}}
		\mathbf{e}^{\rm meas}_{i,k} & = c  \sum_{j\in(i\cup\N_{i,k})} \alpha_{ij,k} \H_{j,k}^T\R\left(\z_{j,k}  - \H_{j,k}\hat{\x}_{i,k}^{-} \right),\label{eq_error_z}\\
		\mathbf{e}^{\text{cons}}_{i,k} & = \sum_{j\in(i\cup\N_{i,k})}\beta_{ij,k}\left(\hat{\x}_{j,k}^{-} -\hat{\x}_{i,k}^{-}\right), \label{eq_error_x}\\
		\S_{i,k} & =c  \sum_{j\in(i\cup\N_{i,k})}\alpha_{ij,k} \H_{j,k}^T\R\H_{j,k}+ \I_6  ,\label{eq_covariance_S}\\
		\intertext{\textbf{Correction:}}
		\hat{\M}_{i,k} & = (\gamma_2\hat{\M}_{i,k}^{-}+ \S_{i,k})^{-1},\label{eq_M_correction}\\
		\hat{\x}_{i,k} & = \hat{\x}_{i,k}^{-}  + \hat{\M}_{i,k} ( \mathbf{e}^{\rm meas}_{i,k} + \mathbf{e}^{\text{cons}}_{i,k}).\label{eq_x_correction}
	\end{align}
	The interpretations of the algorithm are discussed as follows.
	
	\textit{Prediction:}
	While $\hat{\x}_{i,k-1}$ and $\hat{\M}_{i,k-1}$ are for time $k-1$, $\hat{\x}^{-}_{i,k}$ in \eqref{eq_x_pred} and  $\hat{\M}^{-}_{i,k}$ in \eqref{eq_M_pred} are the predicted ones at time $t$ based on the state transition model.
	
	\textit{Innovation:} $\mathbf{e}^{\rm meas}_{i,k}$ in \eqref{eq_error_z} describes the measurements error and $\mathbf{e}^{\text{cons}}_{i,k}$ in \eqref{eq_error_x} denotes the consensus error. They both represent the discrepancy and, hence, innovation information.
	
	\textit{Correction:} The matrix $\hat{\M}_{i,k}$ in \eqref{eq_M_correction} corresponds to the error covariance matrix, which is obtained from its prediction and the innovation covariance matrix $\S_{i,k}$. Then, $\hat{\x}_{i,k}$ is corrected in \eqref{eq_x_correction} based on the prediction and consensus errors.
	
	Finally, to implement the STT algorithm, each observer must receive the following information from its neighbors: $\{\hat{\x}^{-}_{j,k},\z_{j,k},\H_{j,k}\}$ where $j\in\N_{i,k}$. Here, $\{\z_{j,k},\H_{j,k}\}$ correspond to the measurements obtained by observer $j$, and $\hat{\x}^{-}_{j,k}$ is the estimate prediction of observer $j$.
	
	\subsection{Derivation of the STT algorithm}\label{derivation of STT}
	
	We next present the detailed derivation of the STT algorithm proposed in the last subsection.
	
	1) First, we derive the closed-form solution that can minimize the objective function in
	\eqref{eq_J_x_k}.
	In particular, the gradient of $J(\hat{x}_{i,k})$ is
	\begin{align*}
		\nabla_{\hat{\x}_{i,k}}J(\hat{\x}_{i,k}) = 2\left(\sum_{t =1}^k\lambda_t^{(k)}\S_{i,t}^{(k)}\right)\hat{\x}_{i,k}- 2\left(\sum_{t =1}^k\lambda_t^{(k)}\y_{i,t}^{(k)}\right),
	\end{align*}
	where
	\begin{align}
		\y_{i,t}^{(k)}
		&=(\A^{t-k})^{T}\left[ \sum_{j\in (i\cup\N_{i,t})}(c \alpha_{ij,t}\H_{j,t}^T\R \z_{j,t}+ \beta_{ij,t}\hat{\x}_{j,t}^{-})\right],\label{eq_value_y_t}\\
		\S_{i,t}^{(k)}
		& = (\A^{t-k})^{T}\left[\sum_{j\in (i,\N_{i,t})}( c  \alpha_{ij,t}\H_{j,t}^T\R \H_{j,t})+ \I_6 \right]\A^{t-k}.\label{eq_value_S_t}
	\end{align}
	Note that $\y_{i,t}^{(k)}$ and $\S_{i,t}^{(k)}$ depend on both $t$ and $k$. When $t=k$, we have $\S_{i,k}^{(k)}\doteq \S_{i,k}$.
	By solving $\nabla_{\hat{\x}_{i,k}}J(\hat{\x}_{i,k})=0$, we can obtain the closed-form solution as
	\begin{align}
		\hat{\x}_{i,k} = \left(\sum_{t =1}^k\lambda_t^{(k)}\S_{i,t}^{(k)}\right)^{-1}\left(\sum_{t =1}^k\lambda_t^{(k)}\y_{i,t}^{(k)}\right).\label{eq_opt_x}
	\end{align}
	It should be noted that \eqref{eq_opt_x} cannot be implemented in practice because it involves all the historical information before time step $k$. We need to derive a recursive expression.
	
	2) Second, we derive the recursive expression of the optimal solution in \eqref{eq_opt_x}.
	To do that, define
	\begin{align}
		\hat{\M}_{i,k}
		&\doteq \lambda_{k}^{(k)}\left(\sum_{t =1}^k\lambda_t^{(k)}\S_{i,t}^{(k)}\right)^{-1},\label{eq_value_hat_M_i_k}\\
		\bar{\y}_{i,k}
		&\doteq \frac{1}{\lambda_{k}^{(k)}}\sum_{t =1}^k\lambda_t^{(k)}\y_{i,t}^{(k)}.\nonumber
	\end{align}
	Then, $\hat{\x}_{i,k}$ in \eqref{eq_opt_x} can be rewritten as $\hat{\x}_{i,k}  =   \hat{\M}_{i,k}\bar{\y}_{i,k}$.
	We next derive the recursive forms of $\hat{\M}_{i,k}$ and $\bar{\y}_{i,k}$, respectively, to obtain the recursive form of $\hat{\x}_{i,k}$.
	First of all, the relationship between $\S_{i,t}^{(k-1)}$ and $\S_{i,t}^{(k)}$, $\lambda_{t}^{(k-1)}$ and $\lambda_{t}^{(k)}$, and $\y_{i,t}^{(k-1)}$ and $\y_{i,t}^{(k)}$ can be obtained from \eqref{eq_value_S_t}, \eqref{eq_value_lambda}, and \eqref{eq_value_y_t} as
	\begin{align*}
		\S_{i,t}^{(k-1)} & = \A^T\S_{i,t}^{(k)}\A,\\
		\lambda_{t}^{(k-1)} & =  \frac{\|\A\|(1+\gamma_1)}{\gamma_2}\lambda_{t}^{(k)} = \frac{1}{\lambda_{k}^{(k)}\gamma_2}\lambda_{t}^{(k)},\\
		\y_{i,t}^{(k-1)} & = \A^T\y_{i,t}^{(k)}.
	\end{align*}
	Then, $\hat{\M}_{i,k-1}$ can be rewritten as
	\begin{align*}
		\hat{\M}_{i,k-1} &= \lambda_{k-1}^{(k-1)}\left(\sum_{t =1}^{k-1}\lambda_t^{(k-1)}\S_{i,t}^{(k-1)}\right)^{-1}\\
		&= \gamma_2\left(\lambda_{k-1}^{(k-1)}\right)^2\left(\sum_{t =1}^{k-1}\lambda_t^{(k)}\A^T\S_{i,t}^{(k)}\A\right)^{-1}\\
		&=\gamma_2\left(\lambda_{k-1}^{(k-1)}\right)^2\A^{-1}\left(\sum_{t =1}^{k-1}\lambda_t^{(k)}\S_{i,t}^{(k)}\right)^{-1}(\A^T)^{-1}.
	\end{align*}
	Since $\lambda_{k-1}^{(k-1)} = 1/(\|\A\|(1+\gamma_1)) = \lambda_{k}^{(k)}$ by definition, taking the matrix inverse on both sides of the above equation gives
	\begin{align*}
		\sum_{t =1}^{k-1}\lambda_t^{(k)}\S_{i,t}^{(k)}
		&= \gamma_2\left(\lambda_{k-1}^{(k-1)}\right)^2(\A\hat{\M}_{i,k-1}\A^T)^{-1}\\
		&= \gamma_2\left(\lambda_{k}^{(k)}\right)^2(\A\hat{\M}_{i,k-1}\A^T)^{-1}.
	\end{align*}
	Substituting the above equation into \eqref{eq_value_hat_M_i_k} yields
	\begin{align}
		\hat{\M}_{i,k}
		&  =\lambda_{k}^{(k)}\left(\sum_{t =1}^{k}\lambda_t^{(k)} \S_{i,t}^{(k)}\right)^{-1}\nonumber\\
		& = \left[\frac{1}{\lambda_{k}^{(k)}}\left(\sum_{t =1}^{k-1}\lambda_t^{(k)}\S_{i,t}^{(k)}\right) + \S_{i,k}^{(k)}\right]^{-1}\nonumber\\
		& =\left[\gamma_2\lambda_{k}^{(k)}(\A\hat{\M}_{i,k-1}\A^T)^{-1} + \S_{i,k}^{(k)}\right]^{-1}\nonumber\\
		& =  \left(\gamma_2\hat{\M}_{i,k}^{-} + \S_{i,k}^{(k)} \right)^{-1},\label{eq_rel_S_M}
	\end{align}
	where $\hat{\M}_{i,k}^{-}$ is given in \eqref{eq_M_pred}.
	Similarly, the relationship between $\y_{i,t}^{(k-1)}$ and $\y_{i,t}^{(k)}$ can be obtained from \eqref{eq_value_y_t}:
	The recursive form of $\bar{\y}_{i,k}$ can be derived as
	\begin{align}
		\bar{\y}_{i,k}
		& = \frac{1}{\lambda_{k}^{(k)}} \sum_{t =1}^{k}\lambda_t^{(k)} \y_{i,t}^{(k)}\nonumber\\
		& =  \frac{1}{\lambda_{k}^{(k)}}\sum_{t =1}^{k-1}\lambda_t^{(k)}\y_{i,t}^{(k)}+ \y_{i,k}^{(k)}\nonumber\\
		& =    \gamma_2(\A^T)^{-1}\sum_{t =1}^{k-1}\lambda_t^{(k-1)}\y_{i,t}^{(k-1)}+ \y_{i,k}^{(k)}\nonumber\\
		&= \gamma_2(\A^T)^{-1}\bar{\y}_{i,k-1}  +  \y_{i,k}^{(k)}.\label{eq_bar_y}
	\end{align}
    
	3) Third, since $\S_{i,k}^{(k)}$ and $\hat{\M}_{i,k}^{-} $ are symmetric positive definition matrices, \eqref{eq_rel_S_M} can be rewritten as
	\begin{align}
		\hat{\M}_{i,k} =  \left[\I_6  - \left(\gamma_2 (\S_{i,k}^{(k)})^{-1} \hat{\M}_{i,k}^{-}+ \I_6  \right)^{-1} \right]\frac{1}{\gamma_2}\left(\hat{\M}_{i,k}^{-}\right)^{-1},\label{eq_UCV_M}
	\end{align}
	based on the preliminary fact in \eqref{eq_tool_ucv}.
	Substituting \eqref{eq_bar_y} and \eqref{eq_UCV_M} into $\hat{\x}_{i,k}=\hat{\M}_{i,k}\bar{\y}_{i,k}$ yields the recursive form of $\hat{\x}_{i,k}$:
	\begin{align*}
		\hat{\x}_{i,k}
		& = \A \hat{\x}_{i,k-1} -   \hat{\M}_{i,k} \S_{i,k}^{(k)} \A\hat{\x}_{i,k-1} + \hat{\M}_{i,k}\y_{i,k}^{(k)}\\
		&= \hat{\x}_{i,k}^{-} + \hat{\M}_{i,k}\left(  \y_{i,k}^{(k)}- \S_{i,k}^{(k)} \hat{\x}_{i,k}^{-}\right),
	\end{align*}
	where $\hat{\x}_{i,k}^{-} $ is given in \eqref{eq_x_pred} and $ \y_{i,k}^{(k)} - \S_{i,k}^{(k)} \hat{\x}_{i,k}^{-}$ can be split into
	\begin{align*}
		& \y_{i,k}^{(k)}- \S_{i,k}^{(k)} \hat{\x}_{i,k}^{-} \\
		& =  \sum_{j\in (i\cup\N_{i,k})}(c \alpha_{ij,t}\H_{j,k}^T\R \z_{j,k}+ \beta_{ij,k}\hat{\x}_{j,k}^{-}) \\
		& \qquad - \sum_{j\in (i\cup\N_{i,k})}( c  \alpha_{ij,k}\H_{j,k}^T\R \H_{j,k})\hat{\x}_{i,k}^{-} -  \hat{\x}_{i,k}^{-}\\
		& = \mathbf{e}^{\rm meas}_{i,k} + \mathbf{e}^{\text{cons}}_{i,k},
	\end{align*}
	where $\mathbf{e}^{\rm meas}_{i,k}$ and $\mathbf{e}^{\text{cons}}_{i,k} $ are given in \eqref{eq_error_z} and \eqref{eq_error_x}, respectively.
	
	\section{Convergence Analysis}\label{sec_covergence_analysis}
	
	The STT algorithm in \eqref{eq_x_pred}-\eqref{eq_x_correction} is executed by a single observer. Since each observer's estimation relies on other observers' estimation, we need to prove that the overall interconnected system is convergent.
	The aim of the convergence analysis is to show that the expectation of the error converges to zero.
	
	Suppose the current time step is $k$. The estimation error of observer $i$ is denoted as $\Eta_{i,k}\doteq\hat{\x}_{i,k} - \x_k $. It follows from \eqref{eq_opt_x} that
	\begin{align*}
		\Eta_{i,k}
		& =  \hat{\x}_{i,k} - \x_k\\
		& =\left(\sum_{t =1}^k\lambda_t^{(k)}\S_{i,t}^{(k)}\right)^{-1}\sum_{t =1}^k\lambda_t^{(k)}\y_{i,t}^{(k)} - \x_k\\
		& =  \left(\sum_{t =1}^k\lambda_t^{(k)}\S_{i,t}^{(k)}\right)^{-1}\left[\sum_{t =1}^k\lambda_t^{(k)}\left(\y_{i,t}^{(k)} - \S_{i,t}^{(k)} \x_k\right)\right],
	\end{align*}
	where $\y_{i,t}^{(k)}$ and $\S_{i,t}^{(k)}$ are given in \eqref{eq_value_y_t} and \eqref{eq_value_S_t}, respectively.
	
	We next derive the expression of the expectation of $\Eta_{i,k}$.
	
	\begin{lemma}[{Expression of $\E[\Eta_{i,k}]$}]\label{lemma:exp_eta}
		The expectation of $\Eta_{i,k}$ is expressed as
		\begin{align}
			\E[\Eta_{i,k}]= \left(\sum_{t =1}^k\lambda_t^{(k)}\S_{i,t}^{(k)}\right)^{-1}\E[\bar{\Eta}_{i,k}],\label{eq_exp_eta}
		\end{align}
		where
		\begin{align}
			\bar{\Eta}_{i,k} & \doteq \sum_{t =1}^k\left[\lambda_t^{(k)}(\A^{t-k})^{T}\sum_{j\in (i\cup\N_{i,t})}(\beta_{ij,t}\A\Eta_{j,t-1})\right].\label{eq_bar_eta}
		\end{align}
	\end{lemma}
	\begin{proof}
		See Appendix~\ref{proof:exp_eta}.
	\end{proof}
	
	To prove that $\E[\Eta_{i,k}]$ converges to zero, we analyze its upper bound.
	It follows from \eqref{eq_exp_eta} that
	\begin{align}
		\|\E[\Eta_{i,k}] \|
		& \leq \begin{Vmatrix}
			\left(\sum_{t =1}^k\lambda_t^{(k)} \S_{i,t}^{(k)}\right)^{-1}
		\end{Vmatrix}\|\E[\bar{\Eta}_{i,k}]\|\nonumber\\
		& = \frac{1}{\smin\left(\sum_{t =1}^k\lambda_t^{(k)}\S_{i,t}^{(k)}\right)}\|\E[\bar{\Eta}_{i,k}]\|.\label{eq_E_eta_leq_E_bar_eta}
	\end{align}
	
	In the following, we establish Lemma~\ref{lemma:exp_lambda_S} to Lemma~\ref{lemma:S_inv} to prove that
	\begin{align}\label{eq_keyInequalityForConvergence}
		\smin\left(\sum_{t =1}^k\lambda_t^{(k)}\S_{i,t}^{(k)}\right)\ge1.
	\end{align}
	This fact is nontrivial to prove, and it is a unique challenge to analyze bearing-only cooperative estimation.
	Although $\S_{i,t}^{(k)}$ is positive definite as shown in \eqref{eq_value_S_t}, given the expression of $\lambda_t^{(k)}$ in \eqref{eq_value_lambda}, we are not able to trivially claim that \eqref{eq_keyInequalityForConvergence} is valid because $\H_{j,t}^T\R \H_{j,t}$ in \eqref{eq_value_S_t} is rank deficient. We must carefully explore the intrinsic structure of $\H_{j,t}$ and analyze the properties of orthogonal projection matrices embedded in there.
	
	\begin{lemma}[Expression of $\sum_{t =1}^k\lambda_{t}^{(k)}\S_{i,t}^{(k)}$]\label{lemma:exp_lambda_S}
		The expression of $\sum_{t =1}^k\lambda_{t}^{(k)}\S_{i,t}^{(k)}$ is given by
		\begin{align}
			\sum_{t =1}^k\lambda_{t}^{(k)}\S_{i,t}^{(k)}& =\frac{c }{\sigma_{\boldsymbol{\nu}}^2}\sum_{t =1}^k\lambda_{t}^{(k)}\mathbf{F}_t^{(k)}\otimes\bar{\P}_{i,t} \nonumber\\
			&\qquad + \sum_{t =1}^k\lambda_{t}^{(k)}(\A^{t-k})^{T}\A^{t-k},\label{eq_S_FP}
		\end{align}
		where
		\begin{align}
			\mathbf{F}_t^{(k)} \doteq \begin{bmatrix}
				1 & (t-k)\Delta t \\ (t-k)\Delta t& (t-k)^2\Delta t^2
			\end{bmatrix},\label{eq_F}
		\end{align}
		and
		\begin{align}
			\bar{\P}_{i,t} & \doteq \sum_{j\in (i\cup\N_{i,t})}\alpha_{ij,t}\P_{j,t}.\label{eq_bar_P}
		\end{align}
		Here, $\P_{i,t}$ is a concise notation of $\P_{\tilde{\g}_{i,t}}$, where $\tilde{\g}_{i,t}$ is the bearing vector measured by observer $i$ at time step $t$. The matrix operator $\P$ is defined in \eqref{eq_P_g_equation}.
	\end{lemma}
	\begin{proof}
		See Appendix~\ref{proof_exp_lambda_S}.
	\end{proof}
	
	Lemma~\ref{lemma:exp_lambda_S} indicates that, to analyze $\smin(\sum_{t =1}^k\lambda_t^{(k)}\S_{i,t}^{(k)})$, we need first analyze $\smin(\bar{\P}_{i,t})$ and $\smin(\sum_{t =1}^k\lambda_t^{(k)}\mathbf{F}_t^{(k)})$.
	On the one hand, we analyze the lower bound of $\smin(\bar{\P}_{i,t})$. The next is a useful result for the orthogonal projection matrices defined in \eqref{eq_bar_P}.
	
	\begin{lemma}[Value of $\smin(\P_{i,t} +\P_{j,t}) $]\label{lemma:sigma_min_P_12}
		Suppose that $\g_{i,t}$ and $\g_{j,t}$ are two bearing vectors and $\theta_{ij,t}\in[0,\pi)$ is the angle between the two vectors. If $\P_{i,t}=\I_3-\g_{i,t}\g_{i,t}^T$ and $\P_{j,t}=\I_3-\g_{j,t}\g_{j,t}^T$, then
		\begin{align*}
			\smin(\P_{i,t} +\P_{j,t}) & = 1 - |\cos\theta_{ij,t}|.
		\end{align*}
	\end{lemma}
	\begin{proof}
		See Appendix~\ref{proof:sigma_min_P_12}.
	\end{proof}
	
	Lemma~\ref{lemma:sigma_min_P_12} reveals the unique property of the orthogonal projection matrices. It indicates that, when the two bearings are not parallel, the matrix $(\P_{i,t} +\P_{j,t})$ is non-singular. Intuitively, this property corresponds to the triangulation geometry that observers must observe from different directions. It is essential to prove the convergence of the algorithm.
	
	With Lemma~\ref{lemma:sigma_min_P_12}, we can analyze the lower bound of $\smin(\bar{ \P}_{i,t})$.
	
	\begin{lemma}[Lower bound of $\smin(\bar{ \P}_{i,t})$]\label{lemma:sigma_min_bar_P}
		For observer $i$, suppose there exists $\upalpha_{0}>0$ such that
		\begin{align}
			\alpha_{ij,t}\geq\upalpha_{0} , \quad\text{for all $j\in(i \cup \N_{i,t})$.}\label{eq_value_alpha}
		\end{align}
		If there exists $j\in (i\cup\N_{i,t})$ such that the angle $\theta_{ij,t}$ between $\tilde{ \g}_{i,t}$ and $\tilde{ \g}_{j,t}$ satisfies
		\begin{align}
			0<\uptheta_{0}\le\theta_{ij,t}\le \pi-\uptheta_{0}<\pi,\label{eq_theta_12}
		\end{align}
		where $\uptheta_{0}\in(0,\pi/2)$, then we have
		\begin{align*}
			\smin(\bar{\P}_{i,t}) \geq  \upalpha_{0} (1-\cos\uptheta_{0}).
		\end{align*}
	\end{lemma}
	\begin{proof}
		It follows from the definition of $\bar{\P}_{i,t}$ that
		\begin{align*}
			&\smin(\bar{\P}_{i,t})\\
			&\geq\upalpha_{0} \smin\left(\sum_{j\in (i\cup\N_{i,t})}\P_{j,t}\right)\quad\text{(due to $\alpha_{ij,t}\ge\upalpha_0$)}\\
			& \geq \upalpha_{0} \smin\left( \P_{i,t} +\P_{j,t}\right)\quad\quad \text{(only consider $i,j$)}\\
			& = \upalpha_{0} (1 - |\cos\theta_{ij,t}|)\quad\quad \text{(by Lemma~\ref{lemma:sigma_min_P_12})}\\
			& \geq \upalpha_{0} (1 - \cos\uptheta_{0}),
		\end{align*}
		where last inequality is because $|\cos\theta_{ij,t}|\geq \cos\uptheta_{0}$ due to \eqref{eq_theta_12}.
	\end{proof}
	
	On the other hand, we analyze $\smin\left(\sum_{t =1}^k\lambda_t^{(k)}\mathbf{F}_t^{(k)}\right)$.
	
	\begin{lemma}[Lower bound of $\smin\left(\sum_{t =1}^k\lambda_t^{(k)}\mathbf{F}_t^{(k)}\right)$ ]\label{lemma:sigma_F}
		It holds that
		\begin{align}
			\smin\left(\sum_{t =1}^k\lambda_t^{(k)}\mathbf{F}_t^{(k)} \right) \geq  \lambda_{k-1}^{(k)}f(\Delta t),\label{eq_sigma_min_f}
		\end{align}
		where $f(\Delta t)=\left(2+\Delta t^2-\sqrt{4 + \Delta t^4}\right)/2$.
	\end{lemma}
	
	\begin{proof}
		See Appendix~\ref{proof:sigma_F}.
	\end{proof}
	
	With Lemma~\ref{lemma:sigma_min_bar_P} and Lemma~\ref{lemma:sigma_F}, we are ready to analyze $\smin(\sum_{t =1}^k\lambda_t^{(k)}\S_{i,t}^{(k)})$.
	
	\begin{lemma}[Lower bound of $\smin(\sum_{t =1}^k\lambda_t^{(k)}\S_{i,t}^{(k)})$]\label{lemma:S_inv}
		Given $\lambda_t^{(k)}$ defined in \eqref{eq_value_lambda}, if $c$ satisfies
		\begin{align}
			c
			& \ge \frac{(\|\A\|(1+\gamma_1) -1)\|\A\|(1+\gamma_1)\sigma_{\boldsymbol{\nu}}^2}{\gamma_2f(\Delta t)\upalpha_{0} (1-\cos\uptheta_{0})},\label{eq_c}
		\end{align}
		then it holds that
		\begin{align*}
			\smin\left(\sum_{t =1}^k\lambda_t^{(k)}\S_{i,t}^{(k)}\right) \geq 1.
		\end{align*}
	\end{lemma}
	
	\begin{proof}
		See Appendix~\ref{proof:S_inv}.
	\end{proof}
	
	Finally, we can prove main result that $\|\E[\Eta_{i,t}]\|$ converges.
	
	\begin{theorem}[Convergence of $\|\E(\Eta_{i,t})\|$]\label{theorme_E_eta}	
		Given $\lambda_{t}^{(k)}$ defined in \eqref{eq_value_lambda}, if \eqref{eq_value_alpha} and \eqref{eq_theta_12} holds and $\gamma_1>\gamma_2$, then $\|\E[\Eta_{i,k}]\|$ converges to zero exponentially fast for all $i \in \V$ as $k\rightarrow\infty$.
	\end{theorem}
	
	\begin{figure*}[t]
		\centering
		\includegraphics[width=0.9\linewidth]{ 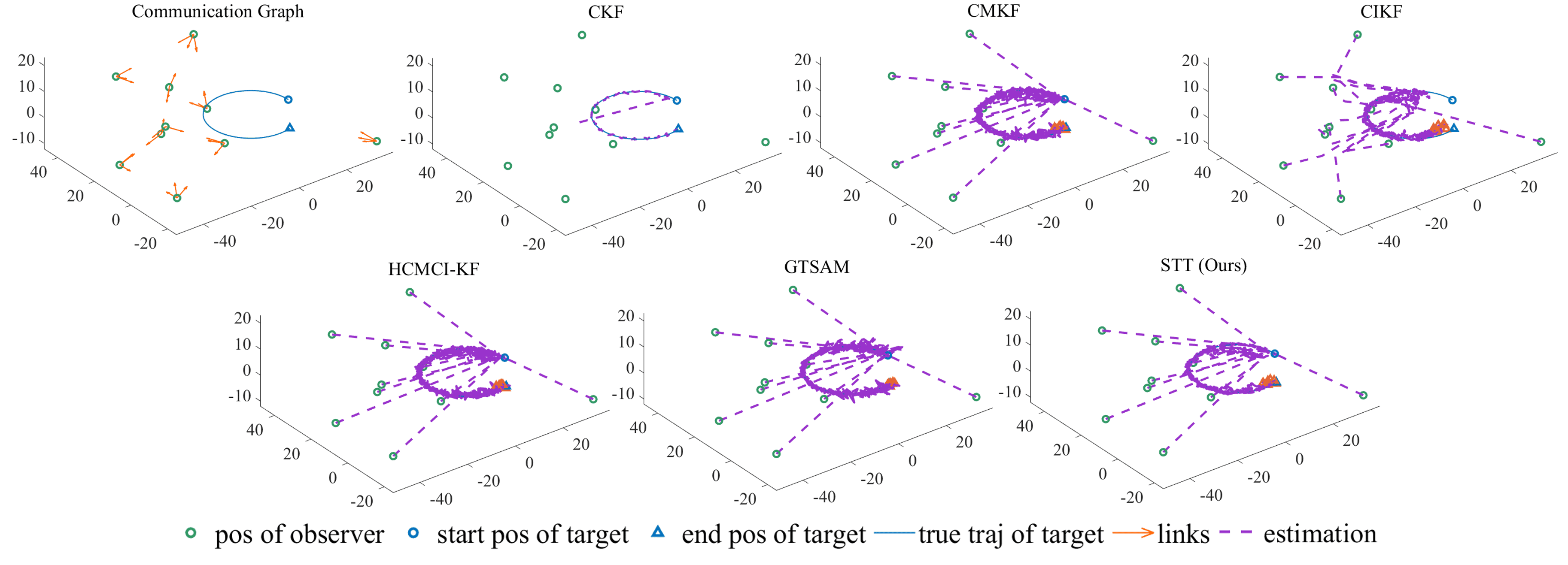}
        \vspace{-0.3cm}
		\caption{Estimation results of CMKF, CIKF, HCMCI-KF, GTSAM, and STT algorithms in a simulation trial. The target moves along a circle. There are 10 observers, and each of them can obtain information from the three closest neighboring observers. The topology of the network is shown in the up-leftmost subfigure.}
		\label{fig_motion_sim}
	\end{figure*}
	
	\begin{proof}\label{proof_7}
		Since $\smin\left(\sum_{t =1}^k\lambda_t^{(k)}\S_{i,t}\right) \geq 1$ according to Lemma~\ref{lemma:S_inv}, it follows from \eqref{eq_E_eta_leq_E_bar_eta} that
		\begin{align}\label{eq_eta_bar_eta}
			\|\E[\Eta_{i,k}]\| \leq \|\E[\bar{\Eta}_{i,k}]\| \leq \max_{j\in\V}\|\E[\bar{\Eta}_{j,k}]\|,
		\end{align}
		where $\V=\{1,\dots,n\}$.
		With the definition of $\bar{\Eta}_{i,k}$ in \eqref{eq_bar_eta}, we have
		\begin{align}
			& \E[\bar{\Eta}_{i,k}]
			 =  \lambda_{k}^{(k)} \sum_{j\in(i\cup\N_{i,k})}\beta_{ij,k-1}\A\E[\Eta_{j,k-1}] \nonumber\\
			&\quad +  \sum_{t=1}^{k-1}\lambda_{t}^{(k)}(\A^T)^{t-k}\sum_{j\in(i\cup\N_{i,t})}\beta_{ij,t-1}\A \E[\Eta_{j,t-1}]\nonumber\\
			& =  \lambda_{k}^{(k)} \sum_{j\in(i\cup\N_{i,k})}\beta_{ij,k-1}\A\E[\Eta_{j,k-1}] \nonumber\\
			& \quad+  \frac{\gamma_2}{\|\A\|(1+\gamma_1)}(\A^T)^{-1}\E[\bar{\Eta}_{i,k-1}],\label{eq_E_bar_Eta}
		\end{align}
  \vspace{-0.3cm}
		where
  \vspace{-0.3cm}
		\begin{align*}
			\E[\bar{\Eta}_{i,k-1}] & = \sum_{t=1}^{k-1}\lambda_{t}^{(k-1)}(\A^T)^{t-(k-1)}\\
			& \quad \cdot \sum_{j\in(i\cup\N_{i,t})}\beta_{ij,t-1}\A \E[\Eta_{j,t-1}].
		\end{align*}
		Taking the norm on both sides of \eqref{eq_E_bar_Eta} gives
		\begin{align*}
			\|\E[\bar{\Eta}_{i,k}]\| &  \leq \frac{1}{(1+\gamma_1)\|\A\|} \sum_{j\in(i\cup\N_{i,k})}\beta_{ij,k}\|\A\|\|\E[\Eta_{j,k-1}]\|\\
			& \quad \quad +\frac{\gamma_2}{(1+\gamma_1)\|\A\|}\|(\A^T)^{-1}\|\|\E[\bar{\Eta}_{i,k-1}]\|.
		\end{align*}
		Since $\sum_{j \in(i\cup\N_{i,k})}\beta_{ij,k} = 1$ and \\$\|\E[\bar{\Eta}_{i,k-1}]\| \leq \max_{j\in\V}\|\E[\bar{\Eta}_{j,k-1}]\|$, substituting \eqref{eq_eta_bar_eta} into the above inequality yields
		\begin{align*}
			\|\E[\bar{\Eta}_{i,k}]\|
			& \le \frac{1}{(1+\gamma_1)}\max_{j\in\V}\|\E[\bar{\Eta}_{j,k-1}]\|\\
			& \qquad  +\frac{\gamma_2}{(1+\gamma_1)\|\A\|}\|(\A^T)^{-1}\|\max_{j\in\V}\|\E[\bar{\Eta}_{j,k-1}]\|.
		\end{align*}
		It can be verified that the matrix $\A$ defined in \eqref{eq_process_matrix_A} satisfies $\|(\A^T)^{-1}\|=\|\A\|$. Then, we have
		\begin{align}
			\|\E[\bar{\Eta}_{i,k}]\|  &  \leq \frac{1}{(1+\gamma_1)}\max_{j\in\V}\|\E[\bar{\Eta}_{j,k-1}]\|\nonumber\\
			& \quad  +\frac{\gamma_2}{(1+\gamma_1)}\max_{j\in\V}\|\E[\bar{\Eta}_{j,k-1}]\|\nonumber\\
			& \leq \frac{1+\gamma_2}{1+\gamma_1}\max_{j\in\V}\|\E[\bar{\Eta}_{j,k-1}]\|.\label{eq_e_bar_eta}
		\end{align}
		Since the above inequality holds for any $i\in\V$ we have
		\begin{align}
			\max_{i\in\V}\|\E[\bar{\Eta}_{i,k}]\|
			& \leq \frac{1+\gamma_2}{1+\gamma_1}\max_{j\in\V}\|\E[\bar{\Eta}_{j,k-1}]\|.\nonumber\\
		\Rightarrow
			\max_{i\in\V}\|\E[\bar{\Eta}_{i,k}]\|    &  \leq \left(\frac{1 + \gamma_2}{1+\gamma_1}\right)^{k-1} \max_{j\in\V}\|\E[\bar{\Eta}_{j,1}]\|.\label{eq_exponentialConvergenceRate}
		\end{align}
  
		If $\gamma_1>\gamma_2$, then $(1+\gamma_2)/(1+\gamma_1)<1$ and hence $\max_{i\in\V}\|\E[\bar{\Eta}_{i,k}]\|\rightarrow0$ as $k\rightarrow\infty$.
		Since $\|\E[{\Eta}_{i,k}]\|\le\max_{i\in\V}\|\E[\bar{\Eta}_{i,k}]\|$ for all $i\in\V$, we know that $\|\E[{\Eta}_{i,k}]\|$ also converges to zero exponentially fast. 		\qed
	\end{proof}
	
	Regarding the convergence rate, it can be seen from \eqref{eq_exponentialConvergenceRate} that the convergence rate is determined by $\gamma_1,\gamma_2$, which are the two parameters in the forgetting factor defined in \eqref{eq_value_lambda}. When $\gamma_2$ is much less than $\gamma_1$, and hence the forgetting factor is small, then the convergence rate is also small, and hence the convergence is fast.
	It is worth mentioning that the network connectivity condition for Theorem~\ref{theorme_E_eta} is mild. It is valid when each agent is linked with at least one neighbor even though the entire network is not fully connected.
    
	It should be noticed that the properties of $\A$ have been used in the proofs of Lemma~\ref{lemma:sigma_F} and Proof~\ref{proof_7}. Specifically, it is required that the conditions in \eqref{eq_sigma_min_f} and \eqref{eq_e_bar_eta} are valid. The state matrix $\A$ can be replaced by other linear models, such as constant acceleration and constant jerk,  and the convergence result is still valid as long as these conditions are satisfied.
 
	\section{Numerical Simulation}\label{sec_simulation}
	
	In this section, the performance of the proposed STT algorithm is evaluated and compared with the benchmark algorithms, including CMKF \cite{olfati2009kalman}, CIKF \cite{olfati2005consensus}, HCMCI-KF \cite{battistelli2014consensus}, and GTSAM \cite{dellaert2012factor}.
	A benchmark algorithm, the centralized Kalman filter (CKF), is also used as a baseline for comparison. We will see that the proposed STT algorithm achieves better performance than the algorithms and comparable performance to the CKF. The simulation code has been put on our GitHub homepage: \href{https://github.com/WestlakeIntelligentRobotics/spatial-temporal-triangulation-for-bearing-only-cooperative-motion-estimation}{https://github.com/WestlakeIntelligentRobotics/spatial-temporal-triangulation-for-bearing-only-cooperative-motion-estimation}.
	
	\subsection{Simulation setup}
	
	Consider a 3D cube with the side lengths of $60\times 60\times 40$.
	In every simulation trial, a set of $n=10$ observers are randomly placed inside the 3D cube. The initial estimate of every observer is selected as $\hat{\x}_{i,0} = \s_{i,0}$.
	Supposing that each observer can always measure the moving target's bearing and each observer can only obtain the information from its three nearest neighbors. The locations of the observers and the network topology are invariant in each simulation trial. They are, however, different across different simulation trials.
	
	The parameters of the STT and DKF algorithms are all optimized by the genetic algorithm (GA) toolbox in Matlab to ensure a fair comparison.
	For the STT algorithm, the optimized parameter values are $c  = 1.8202$, $\gamma_1 = 7.1609$, and $\gamma_2 = 6.1323$. Moreover, $\alpha_{ij,k}  = \beta_{ij,k}= 1/(1+|\N_{i,k}|) = 0.25$ since each observer has three neighbors (i.e., $|\N_{i,k}|=3$). Regarding the GTSAM algorithm, we use the GTSAM library in Matlab. 
	Due to space limitations, the detailed algorithms of the DKFs are omitted here.
	
	\subsection{Evaluation results: Accuracy and convergence rate}
	
	Here, we consider two simulation scenarios.
	In the first scenario, we suppose the target moves along a circle. This scenario is important to study because the target's motion has a nonzero time-varying acceleration and is different from the assumed dynamic model in \eqref{eq_process_matrix_A}. The initial position of the target is $\mathbf{p}_{0} = [15,0,5]^T$, and its velocity is $\mathbf{v}(t) = 5[\sin(t/10\pi),\cos(t/10\pi),0]^T$.
	In the second scenario, we suppose the target moves along a square. It is important to study this scenario because the target has abrupt motion at the corners, which is important to evaluate the algorithms' convergence rate.
	In this scenario, the target's initial position is $\mathbf{p}_{0} = [20,20,5]^T$, and the initial velocity is $\mathbf{v}(0) =[0,-6,0]^T$. The velocity direction rotates clockwise by $90^{\circ}$ every six seconds.
	
	A zero-mean Gaussian noise is added to each bearing measurement based on \eqref{eq_measuredBearingVector} in the simulations. The standard deviation is $0.1~\text{rad} \approx 5.73^{\circ}$.
	In the following, each simulation result is an average of 100 simulation trials.
	An example of the network topology and the estimation trajectories are shown in Fig.~\ref{fig_motion_sim}.
	The estimation errors in terms of average root-mean-square-error (RMSE) are shown in Fig.~\ref{fig_circle_motion_tracking_simulaiton_result} and Fig.~\ref{fig_square_motion_tracking_simulaiton_result}, respectively.
	
    Compared to other benchmark algorithms, STT achieves \emph{smaller} position and velocity estimation errors. 
    The convergence rate is also \emph{faster}. This is important because it indicates that STT can better track maneuverable targets.
	
	\begin{figure*}
		\centering
		\subfloat[Estimation error of different algorithms in the scenario where the target moves along a circle.]{
		\label{fig_circle_motion_tracking_simulaiton_result}
		\includegraphics[width=0.46\linewidth]{ 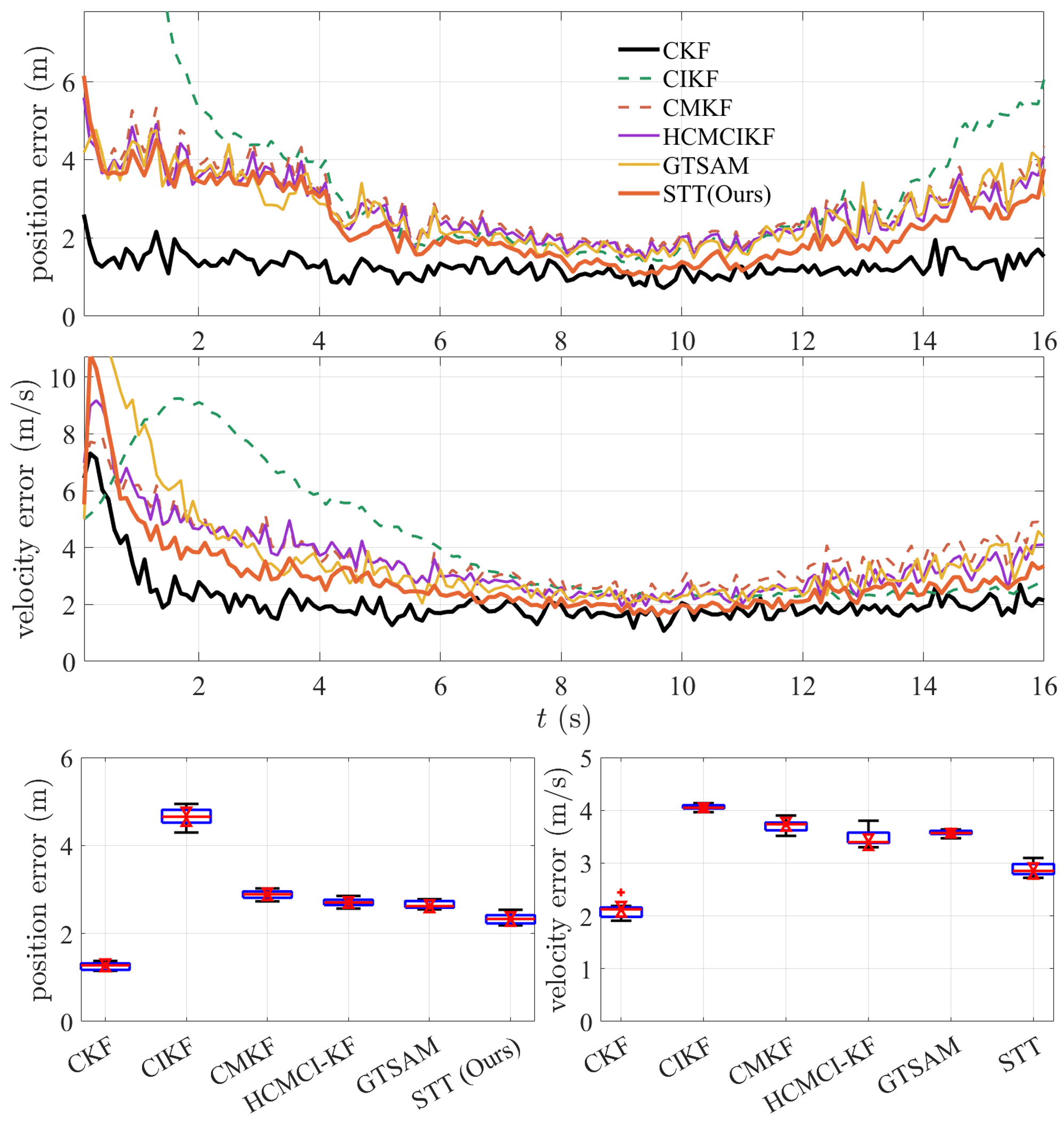}}\quad
		\subfloat[Estimation error of different algorithms in the scenario where the target moves along a square shape.]{
		\label{fig_square_motion_tracking_simulaiton_result}
		\includegraphics[width=0.47\linewidth]{ 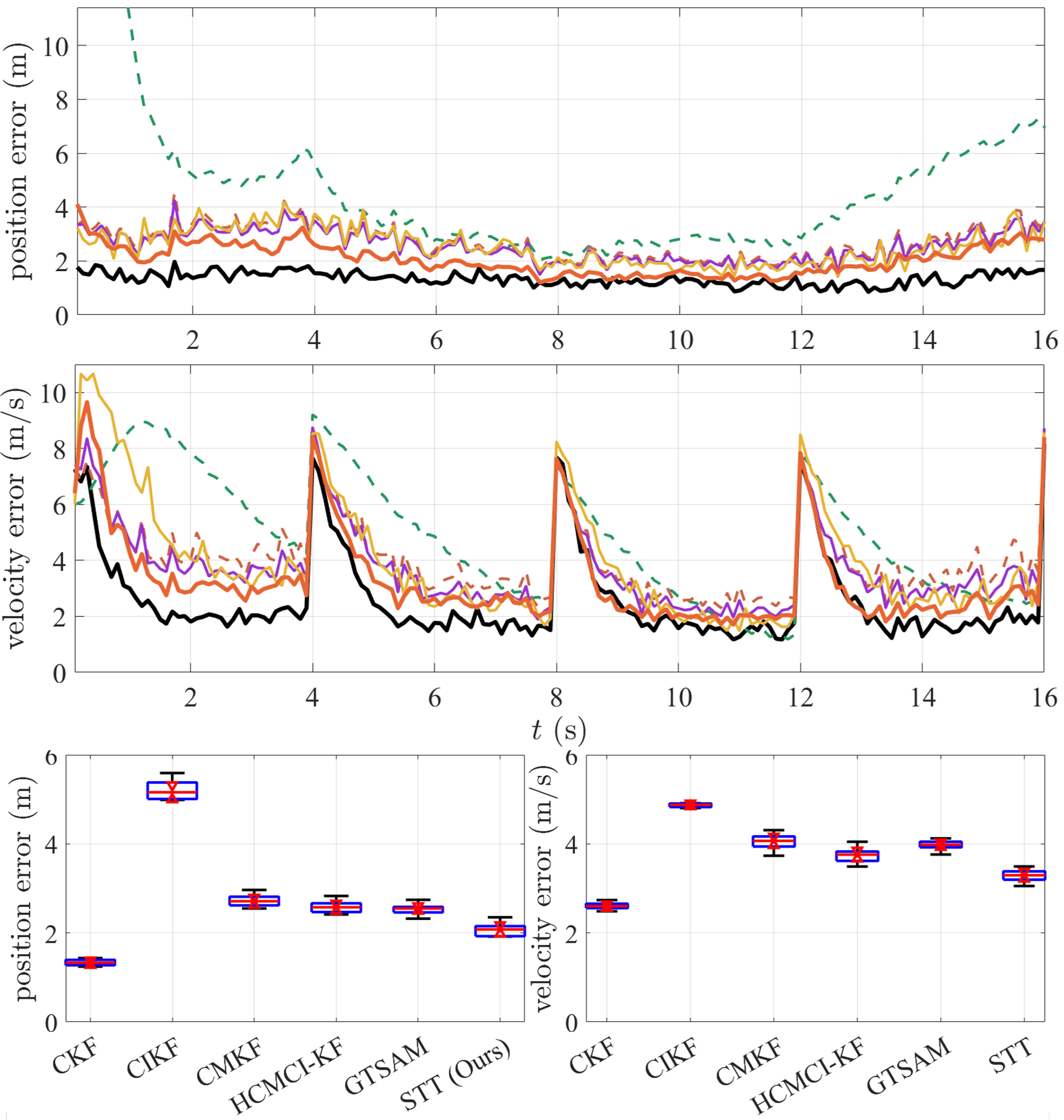}}
	\caption{The performance of different algorithms in estimating the target's state with circle motion and square motion, respectively.}
	\end{figure*}

	\subsection{Evaluation results: Influence of noise}
	
	We now evaluate the influence of the measurement noises on the estimation accuracy.
	We consider the case where the target moves along a square.
	Here, the standard deviation of the zero-mean Gaussian noise added to each bearing measurement varies from $0.01$ rad to $0.3$ rad.
	Each simulation result is an average of 100 times simulation trials.
	
	The simulation results are shown in Fig.~\ref{fig:measure_errors_RMSE}.
	It can be seen that the accuracy for both position and velocity estimation drops as the noise's standard deviation increases. However, STT has the best accuracy across different noise levels compared to the other algorithms.
	
	\begin{figure}[t]
		\centering
        \includegraphics[width=1\linewidth]{ 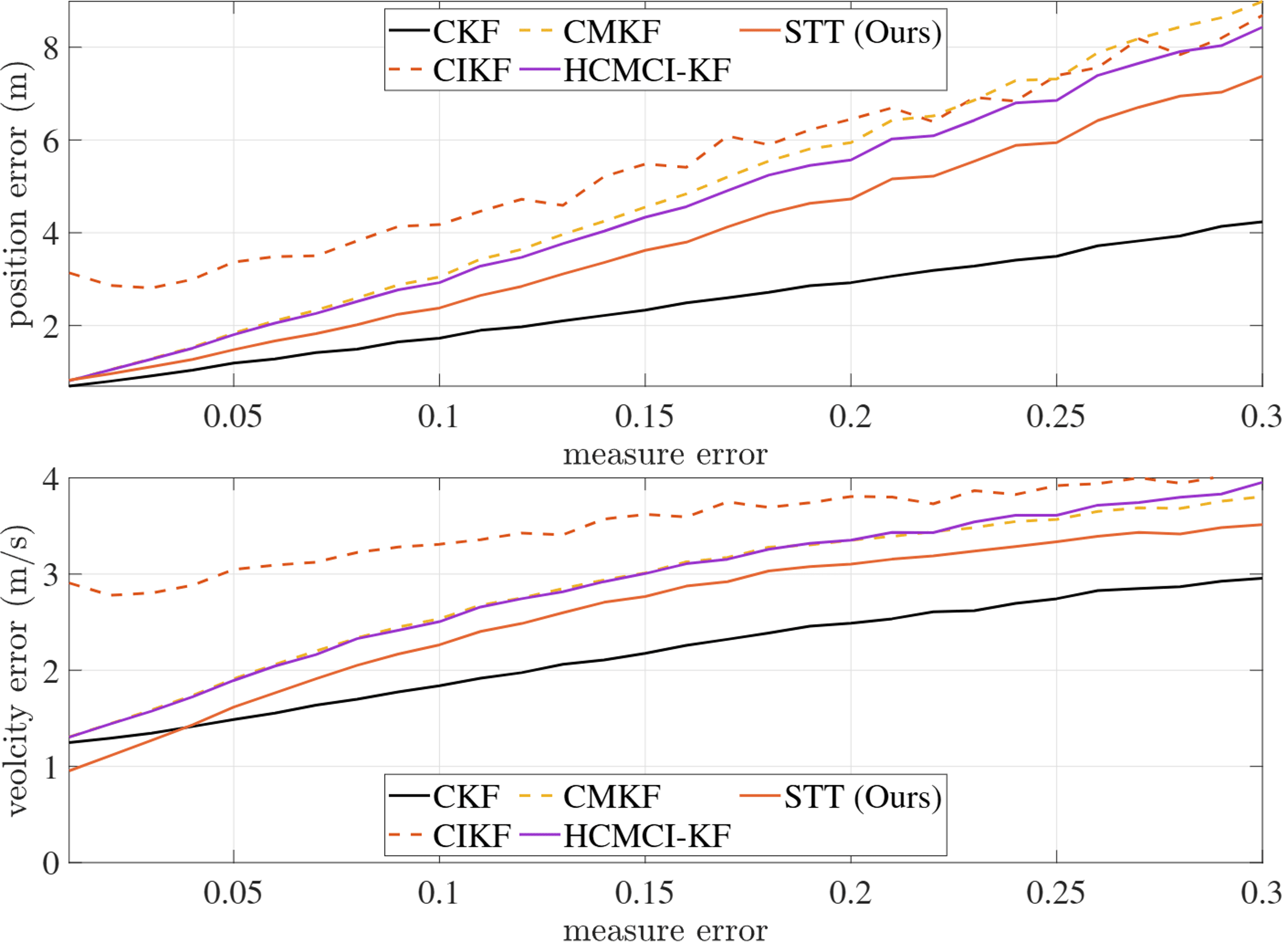}
		\caption{The influence of measure noises on the estimation of different algorithms.}
		\label{fig:measure_errors_RMSE}
	\end{figure}
	
	\subsection{Evaluation results: Computational efficiency}
	
	We further compare the computational efficiency of STT with respect to CIKF, CMKF, HCMCI-KF, and GTSAM. Specifically, in the simulation, the target moves along a circle, and the initial positions of the observers are randomly selected for each simulation trial. We conducted 1,000 simulation trials. The results are presented in Table~\ref{tab_computatioal_time}.
	
	As can be seen from Table~\ref{tab_computatioal_time}, the computational time of the CMKF algorithm is the shortest since it requires the least interaction.
	GTSAM's computational cost exceeds that of other algorithms since it performs optimization in every time step.
	As a comparison, the proposed STT algorithm has slightly greater computational time than CMKF but less than CIKF, HCMCI-KF, and GTSAM. This is because STT requires more information exchange of the estimated state with neighbors than CMKF and has a less computational load of the covariance matrix information than CIKF and HCMCI-KF.
    
\begin{table}[h]
	\caption{The average computational time of different algorithms per time step.}\label{tab_computatioal_time}
	\centering
	\begin{tabular}{c|c}
		\hline
		Algorithms & Computational times (s) \\\hline
		CIKF & 0.66986   \\
		CMKF &  0.16762  \\
		HCMCI-KF & 0.83729   \\
		GTSAM & 14.184\\
		STT & 0.33512 \\ \hline
	\end{tabular}
\end{table}

	\section{Real-World Experiments}\label{sec_experiment}
	
	We implemented the proposed STT algorithm in a real-world vision-based cooperative aerial target pursuit system to verify its effectiveness under practical conditions.
	
	\begin{table}[t]
		\centering
		\caption{Key specifications of the pursuer MAV}
		\begin{tabular}{lll}
			\hline
			Parameter                                       & Value                            & Unit                          \\ \hline
			{\color[HTML]{333333} Dimension}    & {\color[HTML]{333333} 810x670}   & {\color[HTML]{333333} mm}     \\
			{\color[HTML]{333333} Mass}                     & {\color[HTML]{333333} 6.3}       & {\color[HTML]{333333} kg}     \\
			{\color[HTML]{333333} FOV of the gimbal camera} & {\color[HTML]{333333} 82.9}      & {\color[HTML]{333333} degree} \\
			{\color[HTML]{333333} Video resolution}         & {\color[HTML]{333333} 1920x1080} & {\color[HTML]{333333} pix}    \\
			{\color[HTML]{333333} Frame rate}               & {\color[HTML]{333333} 30}        & {\color[HTML]{333333} FPS}    \\
			{\color[HTML]{333333} Communicate rate}         & {\color[HTML]{333333} 24}        & {\color[HTML]{333333} Hz}    \\ \hline
		\end{tabular}
		\label{tab:key_specification}
	\end{table}
	
	The system consists of three pursuer MAVs (DJI M300) and one target MAV (DJI Phantom 4). See Fig.~\ref{fig_purser_system} for illustration.
	Each pursuer MAV uses its onboard camera to detect, locate, and follow the target MAV.
	The key specifications of each pursuer MAV are shown in Table~\ref{tab:key_specification}.
	
	The overall system is fully autonomous, and all the functions are realized onboard. The detection, estimation, and control components form a closed loop, where the former one's output is the latter one's input.
	To the best of our knowledge, it is also the first fully autonomous vision-based cooperative aerial target pursuit system reported in the literature.
	
	\begin{figure}
		\centering
		\subfloat[Three pursuer MAVs (DJI300) and one target MAV (DJI Phantom 4)]{
			\includegraphics[width=\linewidth]{ 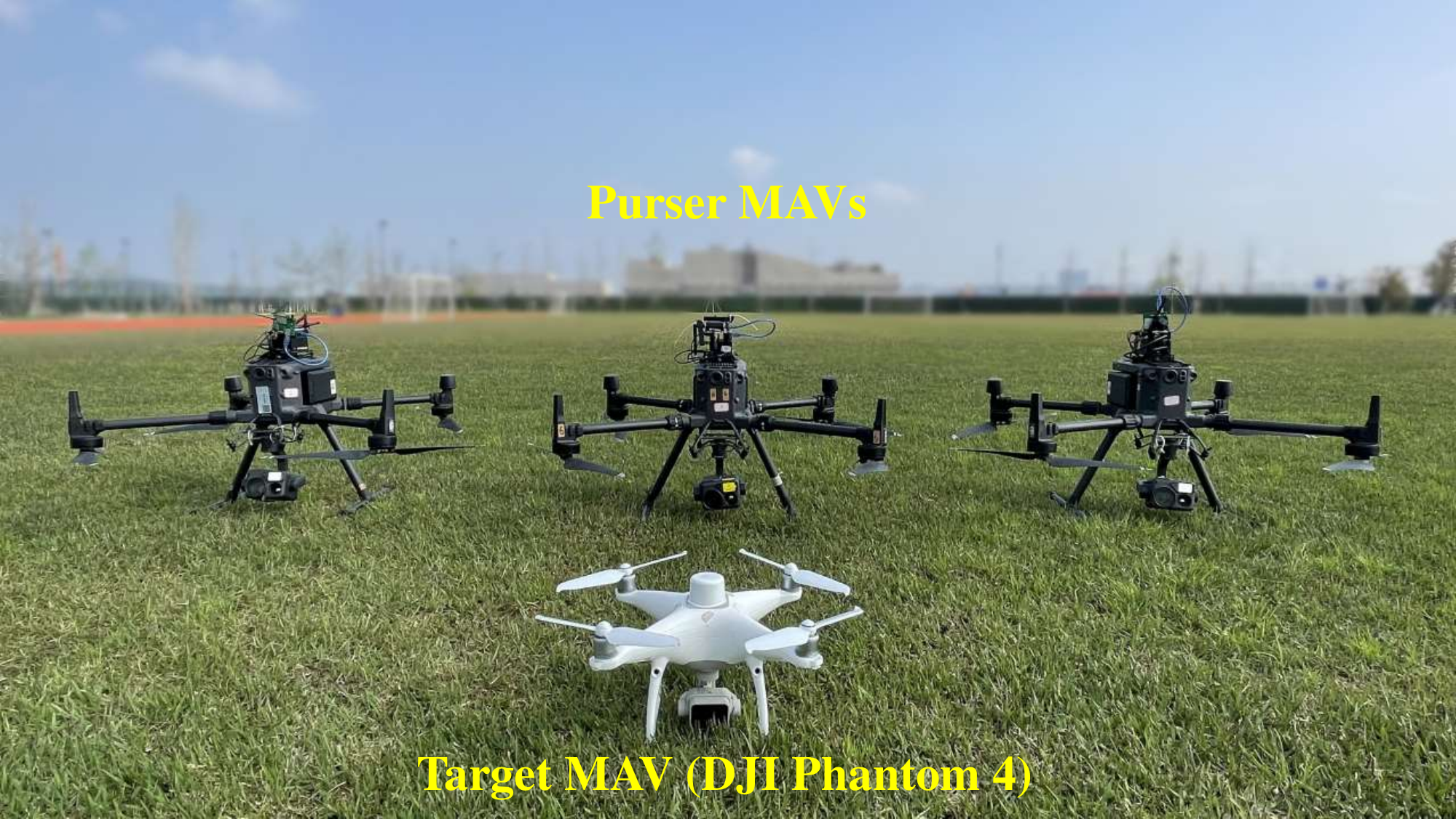}
			\label{fig_purser_system}
		}\\
		\vspace{-0.3cm}
		\subfloat[Onboard components of each pursuer MAV]{
			\includegraphics[width=\linewidth]{ 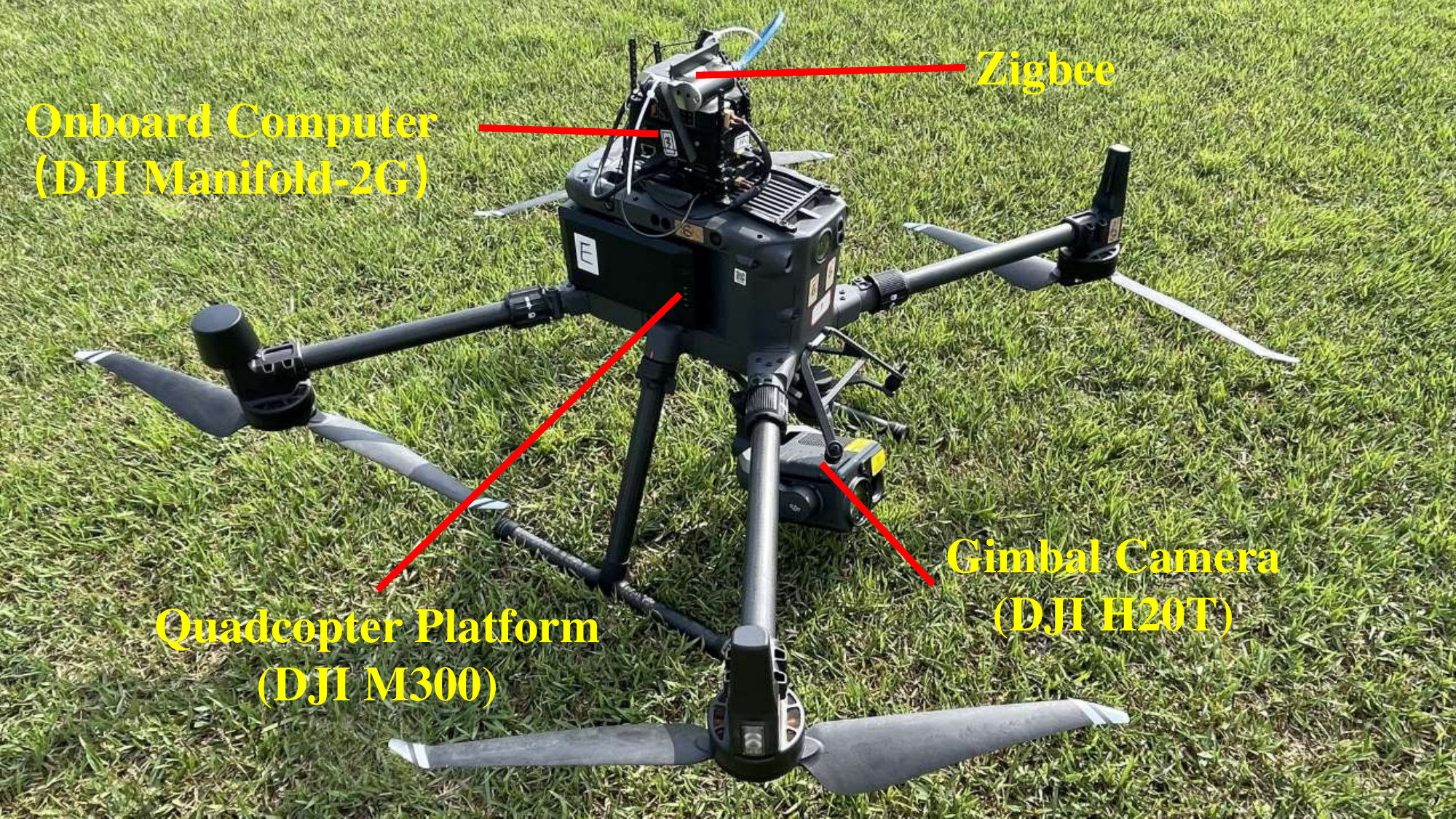}
		}\\
		\vspace{-0.3cm}
		\subfloat[The desired formation shape]{
			\includegraphics[width=1\linewidth, trim=0 70 0 0, clip]{ 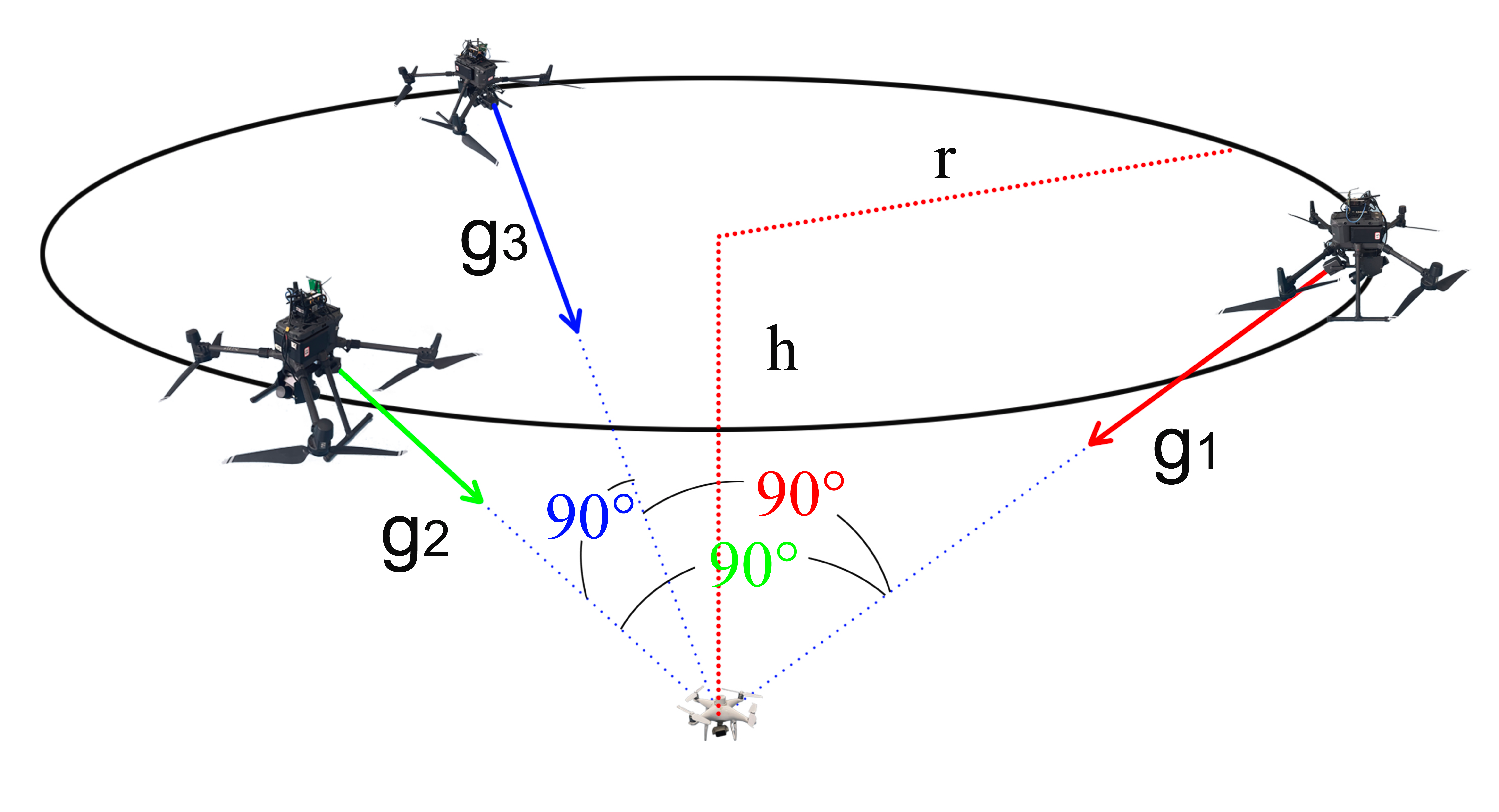}
			\label{fig:capture_MAV_formation}
		}\\
		\vspace{-0.3cm}
		\subfloat[Flight experimental scenario]{
			\includegraphics[width=\linewidth]{ 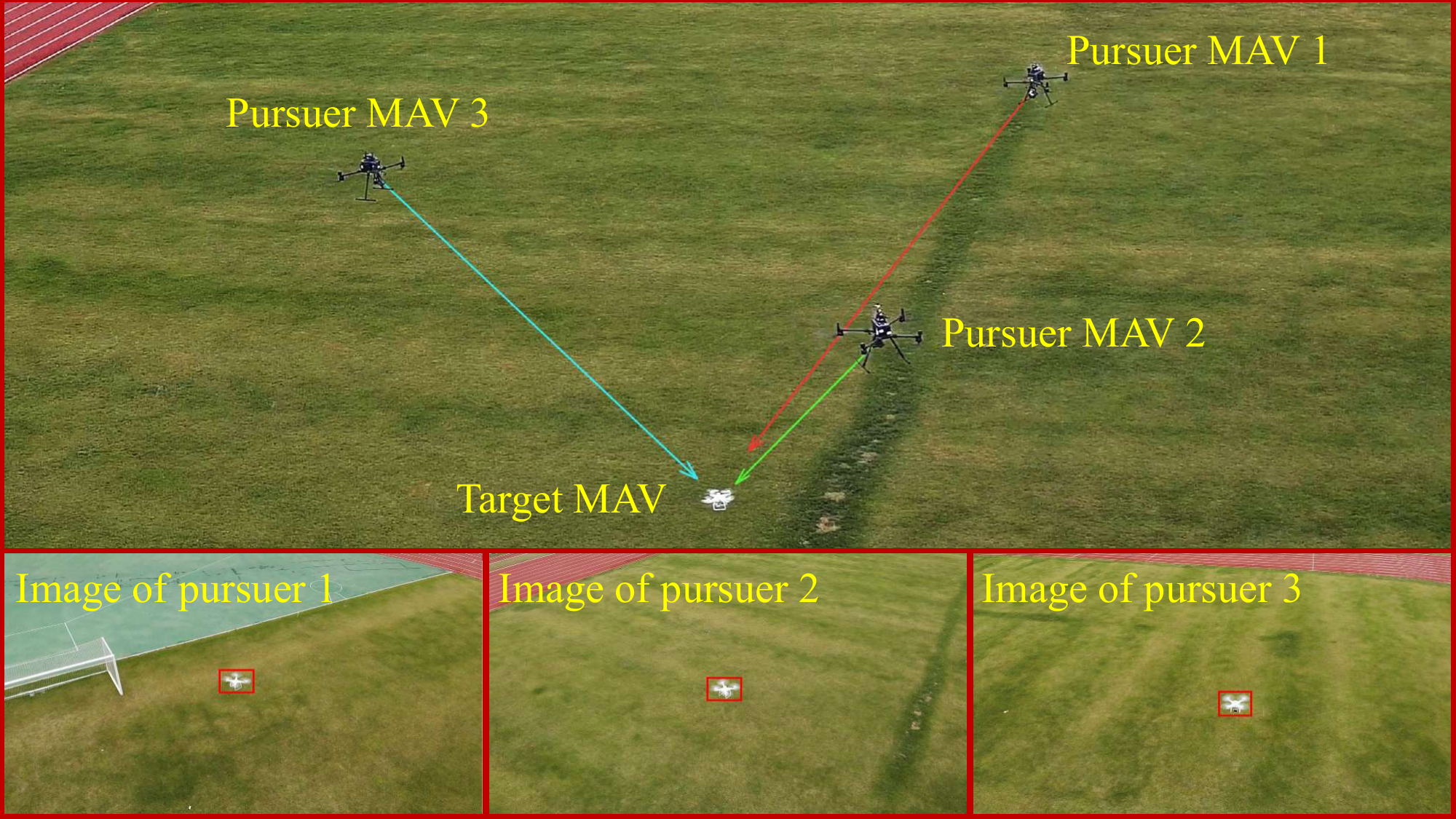}
			\label{fig:fig_experiment_image}
		}\\
		\vspace{-0.3cm}
		\caption{The MAV platforms and flight experimental scenario.}
	\end{figure}
	
	The system architecture of each pursuer MAV consists of the following modules.
	
	1) The first module is a vision-based target detector. Each pursuer MAV has a DJI H20T gimbal camera to search the target MAV.
	We trained a YOLO-based detector (YOLOv5s) using a dataset collected by ourselves. This dataset has more than 20,000 images collected in real-world MAV detection scenarios. We use the YOLO-based detector because it can achieve a good balance between accuracy and speed, as verified by our previous studies \cite{zheng2021air, zheng2022detection}. The accuracy of the trained detector can reach more than 0.9 in the test scenarios.
	The detection frequency is 20~Hz.
	Although the target detector may fail occasionally, the proposed STT algorithm can well handle these non-idealities.
	
	Once the target has been detected in an image, the bearing vector pointing from the purser to the target can be calculated based on its pixel coordinate and the pin-hole camera model \cite[Section~VII]{li2022three}. Additionally, we implement a PI control law to control the gimbal to keep the target at the center of the field-of-view (FOV) of the camera.
	
	2) The second module is state estimation. Once the bearing vector of the target is obtained from vision, each MAV would use the proposed STT algorithm to fuse the information itself and the information shared by its neighbors to estimate the target's motion. The STT algorithm is deployed on the Manifold2G onboard computer, embedded in a robotic operating system (ROS) framework.
	
	3) The third module is wireless communication. During the flight, each pursuer MAV shares information with others via wireless communication based on a Zigbee component. The communication frequency of the Zigbee is 30~Hz. Although the bandwidth of Zigbee is not high, it is sufficient for the proposed STT algorithm because STT only requires a small amount of information to be shared among MAVs.
	
	4) The fourth module is formation control. Once the target's motion has been estimated, the pursuer MAVs would form a desired formation shape to follow the target.
	Fig.~\ref{fig:capture_MAV_formation} shows the desired formation shape, where the target MAV is treated as a leader and the pursuer MAVs as followers. In the desired formation shape, the three purser MAVs should be distributed evenly on a circle with the radius as $r=7$~m at a vertical height of $h=3.5$~m above the target. The horizontal angle between any two pursuer MAVs is 120 degrees. The angle subtended by any two bearing vectors pointing from the pursuers to the target is 90 degrees.
	
	The reason why the desired formation is designed in this way is twofold. First, cooperative bearing-only target localization requires certain observability conditions. Loosely speaking, the observability condition requires the pursuers to observe the target from different directions. Otherwise, if the three pursuers observe the target from the same direction (i.e., the bearings are parallel to each other), then it is impossible to locate the target. It has been proven in our previous work \cite{zhao2013optimal} that it is an optimal configuration to maximize the observability of the target when the angles between the bearings are 90~degrees. Second, since the camera carried by each pursuer is underneath its body, it is convenient for the pursuer to observe the target if this formation shape can be achieved.
	
	Finally, although we have studied distributed formation control extensively in our previous works \cite{zhao2015bearing,zhao2016localizability,zhao2019bearing}, we adopted a simple control law here to achieve the desired shape considering that formation control is not the focus in this paper. In particular, once the target's motion has been estimated, the desired position of each pursuer can be calculated respectively according to the desired formation shape. A simple waypoint tracking control law is used to control each pursuit MAV to track the desired position during flight.
	
	The experimental results are shown in Fig.~\ref{fig:experiment_result}.
	In this experiment, the target MAV is remotely controlled by a human pilot to fly randomly.
	The maximum speed and acceleration of the target reach 2.39~m/s and 1.43~$\mathrm{m}/\mathrm{s}^2$, respectively. The trajectories of the purser MAVs and the target in 3D space are shown in Fig.~\ref{fig:trajectory_experiment}.
	The estimation result of target position and velocity by the STT algorithm are given in Fig.~\ref{fig:experiment_error}.
	
	As can be seen from the experimental results, the STT algorithm can effectively estimate the target's motion. The overall closed-loop system involving detection, estimation, and control algorithms can work effectively.
	
	\begin{figure}
		\centering
		\subfloat[Flight trajectories: $\bigtriangleup$ and $\bigcirc$ represent the starting positions of the target MAV and pursuer MAVs, respectively.]{
			\label{fig:trajectory_experiment}
			\includegraphics[width=\linewidth]{ 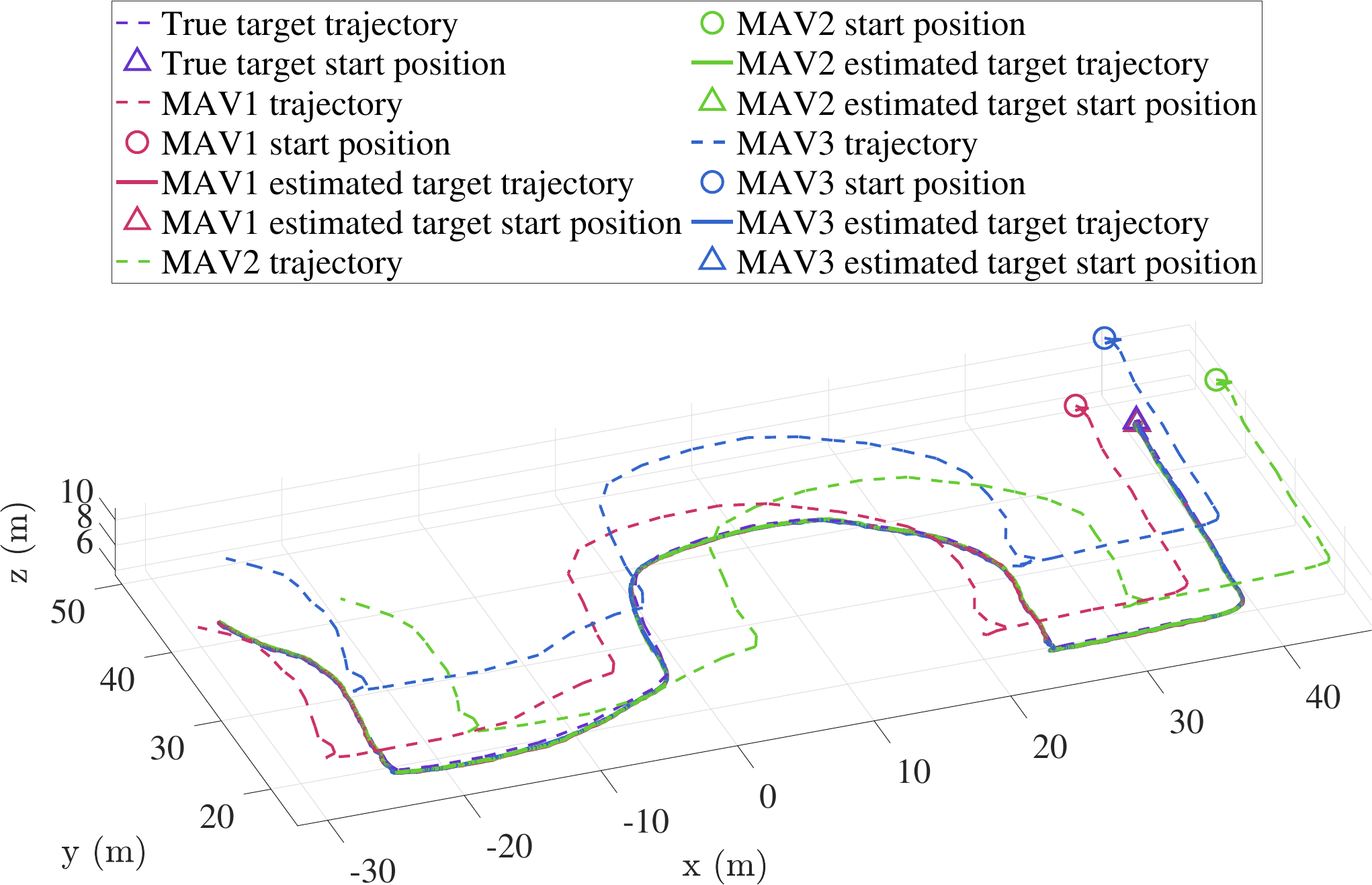}}\\
		\vspace{-0.3cm}
		\subfloat[Ground truth and estimated values of target's position and velocity]{
			\label{fig:experiment_error}
			\includegraphics[width=\linewidth]{ 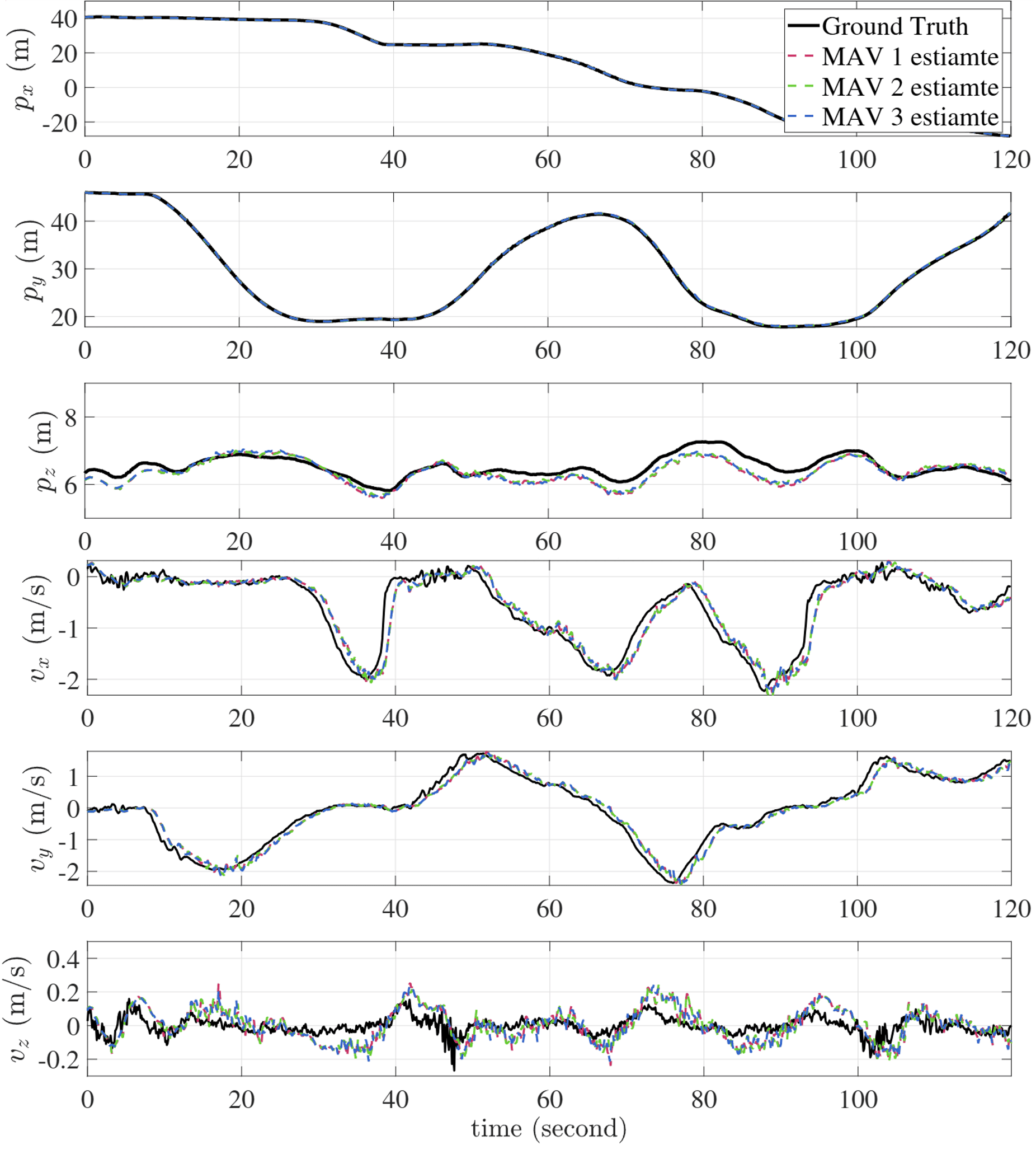}\vspace{-0.3cm}}
		\vspace{-0.3cm}
		\caption{Flight experimental results.}
		\label{fig:experiment_result}
	\end{figure}
	
	\section{Conclusion}\label{sec_conclusion}
	
	This paper proposed a new algorithm named STT for bearing-only cooperative motion estimation. Unlike the conventional algorithms designed based on DKF, STT is designed based on DRLS. Thanks to the objective function that fully incorporates all the available Information and the triangulation geometric constraints, the STT algorithm generates superior estimation performance than the existing DKFs. We also developed a real-world vision-based cooperative aerial target pursuit system to verify the effectiveness of the STT algorithm under practical conditions.
	
	\section{Appendix}\label{sec_appendix}
	
	\subsection{Proof of Lemma~\ref{lemma:exp_eta}}\label{proof:exp_eta}
	
	There are two terms in the expression of $\Eta_{i,k}$. The first term $(\sum_{t =1}^k\lambda_t^{(k)}\S_{i,t}^{(k)})^{-1} $does not contain any random variables. We thus only need to calculate the expectation of the second term $\sum_{t =1}^k\lambda_t^{(k)}(\y_{i,t}^{(k)} - \S_{i,t}^{(k)} \x_k)$.
	
	It follows from the state transition equation \eqref{eq_state_transition} that
	$\x_{k-1} = \A^{-1}\x_{k}  - \A^{-1}\mathbf{w}_{k-1}$.
	Then, for any $t=1,\dots,k$, we have $\x_{t} = \A^{t-k}\x_{k}  - \sum_{\tau=t}^{k-1}\A^{\tau - k }\mathbf{w}_{\tau}$.
	Substituting the above equation into the measurement equation \eqref{eq_pseudo_measurement_equation} gives
	\begin{align*}
  \z_{i,t} -  \H_{i,t}\A^{t-k}\x_{k}  = \boldsymbol{\nu}_{i,t} -\H_{i,t} \sum_{\tau=t}^{k-1}\A^{\tau - k}\mathbf{w}_{\tau}.
	\end{align*}
	Substituting the above equations into $(\y_{i,t}^{(k)} - \S_{i,t}^{(k)} \x_k) $
	gives
	\begin{align}
		& \y_{i,t}^{(k)} - \S_{i,t}^{(k)} \x_k = (\A^{t-k})^{T}\left[\sum_{j\in (i,\N_{i,t})}c \alpha_{ij,t}\H_{j,t}^T\R\boldsymbol{\nu}_{j,t} \right.\nonumber\\
		& \quad -\sum_{j\in (i,\N_{i,t})}c \alpha_{ij,t}\H_{j,t}^T\R\H_{j,t}\left(\sum_{\tau=t}^{k-1}\A^{\tau-k}\mathbf{w}_{\tau}\right) \nonumber\\
		&  \left.+  \mathbf{w}_{t-1}+ \sum_{j \in(i,\N_{i,t})}\beta_{ij,t}\A\Eta_{j,t-1} - \sum_{\tau=t}^{k-1}\A^{\tau-k}\mathbf{w}_{\tau}  \right].\label{eq_y_S}
	\end{align}
	Since $\E[\boldsymbol{\nu}] = 0$ and $\E[\mathbf{w}] = 0$ by assumption, taking expectation on both sides of the above equation gives
	\begin{align*}
		\E\left[\sum_{t =1}^k\left(\y_{i,t}^{(k)} - \S_{i,t}^{(k)} \x_k\right)\right]=\E[\bar{\Eta}_{i,k}],
	\end{align*}
	where $\bar{\Eta}_{i,k}$ is given by \eqref{eq_bar_eta}.
 
	\subsection{Proof of Lemma~\ref{lemma:exp_lambda_S}}\label{proof_exp_lambda_S}
 
	The matrix $\S_{i,t}^{(k)}$ in \eqref{eq_value_S_t} can be rewritten as
	\begin{align*}
		\S_{i,t}^{(k)}
		& =c (\A^{t-k})^{T}\sum_{j \in(i,\N_{i,t})}\left(\H_{j,t}^{T}\R\H_{j,t}+\I_6\right)\A^{t-k} \nonumber\\
		& = c (\A^{t-k})^{T}\begin{bmatrix}
			\frac{1}{\sigma_{\boldsymbol{\nu}}^2}\sum_{j \in(i\cup\N_{i,t})}\P_{j,t}^T\P_{j,t} & \zeros_3 \\
			\zeros_3 & \zeros_3
		\end{bmatrix}\A^{t-k} \nonumber\\
		& \qquad + (\A^{t-k})^{T}\A^{t-k}.
	\end{align*}
	Substituting the expression of $\A$ given in \eqref{eq_process_matrix_A} into the above equation yields
	\begin{align}
		\S_{i,t}^{(k)}
		& = \frac{c }{\sigma_{\boldsymbol{\nu}}^2}\begin{bmatrix}
			\I_3 & \zeros_3 \\ (t-k)\Delta t\I_3  & \I_3
		\end{bmatrix}\begin{bmatrix}
			\bar{\P}_{i,t} & \zeros_3 \\
			\zeros_3 & \zeros_3
		\end{bmatrix}\nonumber\\
		& \qquad \begin{bmatrix}
			\I_3 & (t-k)\Delta t\I_3 \\ \zeros_3 & \I_3
		\end{bmatrix} + (\A^{t-k})^{T}\A^{t-k}\nonumber\\
		& =   \frac{c }{\sigma_{\boldsymbol{\nu}}^2}\begin{bmatrix}
			1 & (t-k)\Delta t \\ (t-k)\Delta t & (t-k)^2\Delta t^2
		\end{bmatrix}\otimes\bar{\P}_{i,t} \nonumber\\
		& \qquad + (\A^{t-k})^{T}\A^{t-k}.\label{eq_S_saperate}
	\end{align}
	By defining \eqref{eq_F}, equation \eqref{eq_S_FP} can be obtained from the above equation.
	
	\subsection{Proof of Lemma~\ref{lemma:sigma_min_P_12}}\label{proof:sigma_min_P_12}
	
	For any $\g_{i,t}$ and $\g_{j,t}$, there always exists a rotation matrix $\R$ so that $\R\g_{i,t}=\g_{i,t}'=[1,0,0]^T$ and $\R\g_{j,t}=\g_{j,t}'=[\cos\theta_{ij,t},\sin\theta_{ij,t},0]^T$.
	The intuition is that $\R$ aligns $\g_{i,t}$ with the x-axis and puts $\g_{j,t}$ in the x-y plane of a new coordinate frame.
	Then, we have
	\begin{align*}
		& \smin(\P_{i,t} + \P_{j,t}) \\
		& = \smin\big[\R(\P_{i,t} + \P_{j,t})\R^T\big]\\
		& = \smin(\I_3 - \g_{i,t}'(\g_{i,t}')^T+\I_3 - \g_{j,t}'(\g_{j,t}')^T)\\
		& = \smin\left(\begin{bmatrix}
			1 -  \cos^2\theta_{ij,t} & \cos\theta_{ij,t}\sin\theta_{ij,t} \\
			\cos\theta_{ij,t}\sin\theta_{ij,t} & 2- \sin^2\theta_{ij,t} \\
		\end{bmatrix}\right).
	\end{align*}
	The singular values of the above two-by-two matrix can be obtained by solving
	\begin{align*}
		(\sigma - 1 +  \cos^2\theta_{ij,t} )&(\sigma - 2+\sin^2\theta_{ij,t}) \\
		& \quad -  \cos^2\theta_{ij,t}\sin^2\theta_{ij,t} = 0.
	\end{align*}
	We can solve the roots as $\sigma = 1\pm \cos\theta_{ij,t}\geq 0$. Then the minimum singular value is
	\begin{align*}
		\smin(\P_{i,t} + \P_{j,t}) & = \min\{1+\cos\theta_{ij,t},1-\cos\theta_{ij,t}\}.
	\end{align*}
	When $\theta_{ij,t}\in[0,\pi/2)$, we have $\smin(\P_{i,t} + \P_{j,t})=1-\cos\theta_{ij,t}$. When $\theta_{ij,t}\in(\pi/2,\pi)$, we have $\smin(\P_{i,t} + \P_{j,t})=1+\cos\theta_{ij,t}$.
	Therefore, the minimum singular value can be written in a unified expression: $\smin(\P_{i,t} + \P_{j,t}) = 1-|\cos\theta_{ij,t}|$ for any $\theta_{ij,t}\in[0,\pi)$.
	
	\subsection{Proof of Lemma~\ref{lemma:sigma_F}}\label{proof:sigma_F}
	
	First of all, we have
	\begin{align*}
		\smin\left(\sum_{t =1}^k\lambda_t^{(k)}\mathbf{F}_t^{(k)}\right)
		&\ge\smin\left(\lambda_{k}^{(k)} \mathbf{F}_k^{(k)} + \lambda_{k-1}^{(k)} \mathbf{F}_{k-1}^{(k)}\right).
	\end{align*}
	Since $\lambda_{k}^{(k)}\geq\lambda_{k-1}^{(k)}$, we further have
	\begin{align}
		\smin\left(\sum_{t =1}^k\lambda_t^{(k)}\mathbf{F}_t^{(k)}\right)
		& \geq \lambda_{k-1}^{(k)}\smin\left(\mathbf{F}_{k}^{(k)} + \mathbf{F}_{k-1}^{(k)}\right)\label{eq_J_F}
	\end{align}
	where
    \vspace{-0.3cm}
	\begin{align*}
		\mathbf{F}_{k}^{(k)} + \mathbf{F}_{k-1}^{(k)} = \begin{bmatrix}
			2& -\Delta t \\ -\Delta t & \Delta t^2
		\end{bmatrix}.
	\end{align*}
	The singular value of $(\mathbf{F}_{k}^{(k)} + \mathbf{F}_{k-1}^{(k)})$ can be obtained by solving $(x - 2)(x-\Delta t^2) - \Delta t^2 = 0$. The roots are
	\begin{align}
		x = \frac{2+\Delta t^2\pm \sqrt{4+\Delta t^4}}{2}\label{eq_root_lambda}.
	\end{align}
	 Substituting the minimum root in \eqref{eq_root_lambda} to \eqref{eq_J_F}
	, the proof is complete.
	
	\subsection{Proof of Lemma~\ref{lemma:S_inv}}\label{proof:S_inv}

	Denote $\P_0 = \alpha_{0}(1-\cos\uptheta_{0})\I_3\in \mathbb{R}^{3\times 3}$. For any $\x$, we have $\x^T(\bar{ \P}_{i,t}  - \P_0) \x\geq\x^T\bar{ \P}_{i,t}\x  -\x \P_0\x  = \x^T\bar{ \P}_{i,t}\x  - \alpha_{0}(1-\cos\uptheta_{0}) \ge \smin(\bar{ \P}_{i,t})- \alpha_{0}(1-\cos\uptheta_{0})\ge0$, where the last inequality is due to Lemma~\ref{lemma:sigma_min_bar_P}.
	Therefore, we know that $(\bar{ \P}_{i,t}  - \P_0)$ is positive semi-definite. Since $\mathbf{F}_t^{(k)}$ is also positive semi-definite, we have
	\begin{align*}
		&\sum_{t =1}^k\lambda_t^{(k)} \mathbf{F}_t^{(k)}\otimes\left( \bar{\P}_{i,t} -  \P_0\right) \geq 0\\
  \Rightarrow&
		\sum_{t =1}^k\lambda_t^{(k)} \mathbf{F}_t^{(k)}\otimes\bar{\P}_{i,t} \geq \sum_{t =1}^k\lambda_t^{(k)} \mathbf{F}_t^{(k)}\otimes\P_0.
	\end{align*}
	It can be further obtained that
	\begin{align*}
		& \smin\left(\sum_{t =1}^k\lambda_t^{(k)} \mathbf{F}_t^{(k)}\otimes\bar{\P}_{i,t} \right)\\
		& \qquad \ge \smin\left[\left(\sum_{t =1}^k\lambda_t^{(k)} \mathbf{F}_t^{(k)}\right)\otimes\P_0\right]\\
		& \qquad= \smin\left(\sum_{t =1}^k\lambda_t^{(k)} \mathbf{F}_t^{(k)}\right)\smin(\P_0)\\
		& \qquad \geq \lambda_{k-1}^{(k)}f(\Delta t)\upalpha_{0}(1-\cos\uptheta_{0}),
	\end{align*}
	where the last inequality is due to Lemma~\ref{lemma:sigma_F} and $\smin(\P_0)=\alpha_{0}(1-\cos\uptheta_{0})$.
	It then follows from \eqref{eq_S_FP} that
	\begin{align}
		& \smin\left(\sum_{t =1}^k\lambda_t^{(k)}\S_{i,t}^{(k)}\right) \nonumber\\
		&  \quad \geq \smin  \left(\lambda_k^{(k)}(\A^{k-k})^{T}\A^{k-k}\right)\nonumber\\
		& \qquad  + \frac{c }{\sigma_{\boldsymbol{\nu}}^2} \smin \left(\sum_{t =1}^k\lambda_t^{(k)} \mathbf{F}_t^{(k)}\otimes \bar{\P}_{j,t}\right)\nonumber\\
		& \quad = \lambda_{k}^{(k)} + \frac{c }{\sigma_{\boldsymbol{\nu}}^2}  \lambda_{k-1}^{(k)}f(\Delta t)\upalpha_{0}(1-\cos\uptheta_{0})
		.\label{eq_sigma_min_S}
	\end{align}
	If
	\begin{align*}
		c  & \geq  \frac{(1-\lambda_k^{(k)})\sigma_{\boldsymbol{\nu}}^2}{\lambda_{k-1}^{(k)}f(\Delta t)\upalpha_{0} (1-\cos\uptheta_{0})}\\
		& = \frac{(\|\A\|(1+\gamma_1)-1)\|\A\|(1+\gamma_1)\sigma_{\boldsymbol{\nu}}^2}{\gamma_2f(\Delta t)\upalpha_{0} (1-\cos\uptheta_{0})},
	\end{align*}
	then substituting $c$, given in \eqref{eq_c}, into \eqref{eq_sigma_min_S} yields
	\begin{align*}
		& \smin	\left(\sum_{t =1}^k\lambda_t^{(k)}\S_{i,t}^{(k)}\right) \\
		& \quad \geq \lambda_{k}^{(k)} 
		+ \frac{c }{\sigma_{\boldsymbol{\nu}}^2}  \lambda_{k-1}^{(k)}f(\Delta t	)\upalpha_{0}(1-\cos\uptheta_{0})  \geq 1.
	\end{align*}
	The proof is complete.
	
	\bibliographystyle{IEEEtran}
	\bibliography{manuscript}
\end{document}